\documentclass{article}

 \usepackage[main, final]{neurips_2026}

\usepackage[utf8]{inputenc} %
\usepackage[T1]{fontenc}    %
\usepackage{hyperref}       %
\usepackage{url}            %
\usepackage{booktabs}       %
\usepackage{amsfonts}       %
\usepackage{nicefrac}       %
\usepackage{microtype}      %
\usepackage{xcolor} 
\usepackage{booktabs}
\usepackage{array}
\usepackage{multirow}
\usepackage{graphicx}
\usepackage{threeparttable}

\usepackage{graphicx}
\usepackage{subfigure}

\usepackage{algorithm}
\usepackage{algorithmic}
\usepackage{natbib}
\usepackage{eso-pic}
\usepackage{forloop}
\usepackage[utf8]{inputenc}
\usepackage[T1]{fontenc} 
\usepackage{booktabs}       %
\usepackage{amsmath,amsbsy,amsgen,amscd,amssymb,amsthm,amsfonts}

\usepackage{mathtools}
\usepackage{bm}
\usepackage{enumitem}
\usepackage{parskip}
\usepackage{threeparttable}
\usepackage{booktabs}
\usepackage{array}
\usepackage{multirow}
\usepackage{graphicx}
\usepackage{threeparttable}
\usepackage{adjustbox}
\usepackage{thmtools}
\usepackage{hyperref}
\usepackage[capitalize, nameinlink]{cleveref}

\newcommand{\R}{\mathbb{R}}

\usepackage{setspace}
\usepackage{tikz}
\usepackage[dvipsnames]{xcolor}
\usepackage{pgfplots}
\usetikzlibrary{arrows.meta,calc,patterns,angles,quotes}
\usepgfplotslibrary{groupplots}
\usepgfplotslibrary{fillbetween}
\pgfplotsset{compat=1.18}
\usepackage{subcaption}

\definecolor{mybrown}{RGB}{150,100,55}
\definecolor{myolive}{RGB}{120,130,0}
\definecolor{mypurple}{RGB}{160,0,160}
\definecolor{mycolor1}{rgb}{0,0.44,0.74}%
\definecolor{mycolor2}{rgb}{0,0.59,0.2}%
\definecolor{mycolor3}{rgb}{0.85,0.325,0.098}%

\DeclareMathOperator*{\argmax}{argmax}

\newcommand{\actionset}[1]{\mathbb{A} ({#1})}
\newcommand{\DrawStripAsym}[5]{%
  \pgfmathsetmacro{\aa}{#1}
  \pgfmathsetmacro{\bb}{#2}
  \pgfmathsetmacro{\cl}{#3}
  \pgfmathsetmacro{\ch}{#4}

  \draw[dashed, line width=1pt , #5]
    ({\ch*\aa - \LineExt*\bb},{\ch*\bb + \LineExt*\aa})
    --
    ({\ch*\aa + \LineExt*\bb},{\ch*\bb - \LineExt*\aa});

  \draw[dashed, line width=1pt, #5]
    ({\cl*\aa - \LineExt*\bb},{\cl*\bb + \LineExt*\aa})
    --
    ({\cl*\aa + \LineExt*\bb},{\cl*\bb - \LineExt*\aa});
}

\theoremstyle{plain}
\newtheorem{theorem}{Theorem}[section]

\newtheorem{proposition}[theorem]{Proposition}
\newtheorem{lemma}[theorem]{Lemma}

\theoremstyle{definition}
\newtheorem{assumption}{Assumption}
\newtheorem{remark}[theorem]{Remark}
\newtheorem{definition}[theorem]{Definition}
\newtheorem{example}{Example}

\usepackage{xspace}

\definecolor{mycolor2}{rgb}{0,0.59,0.2}%

\definecolor{light-gray}{gray}{0.9}
\definecolor{darkblue}{rgb}{0.0,0.0,0.65}
\definecolor{darkred}{rgb}{0.4,0.0,0.0}
\definecolor{mygreen}{RGB}{0,100,0}
\hypersetup{
   colorlinks = true,
   citecolor  = darkblue,
   linkcolor  = darkred,
   filecolor  = darkblue,
   urlcolor   = darkblue,
 }

\definecolor{stripgreen}{RGB}{0,210,0}
\definecolor{stripblue}{RGB}{0,0,220}
\definecolor{stripbrown}{RGB}{140,90,45}
\definecolor{stripyellow}{RGB}{220,205,35} %
\title{Tight Generalization Bounds \\ for Noiseless Inverse Optimization}

\author{%
  $^{1,2\,}$Pouria Fatemi \quad 
  $^{3,4\,}$Hoomaan Maskan \quad 
  $^{1,2\,}$Suvrit Sra \quad 
  $^{5\,}$Peyman Mohajerin Esfahani \\[1em]
  $^{1}\,$Technical University of Munich, School of CIT, Dept. of Mathematics\\[0.1em]
   $^{2}\,$Munich Center for Machine Learning\\[0.1em]
  $^{3}\,$Uppsala University, Dept. of Information Technology\\[0.1em]
  $^{4}\,$Uppsala University, Science for Life Laboratory\\[0.1em]
  $^{5}\,$University of Toronto
}

\begin{document}

\maketitle

\begin{abstract}

Inverse optimization (IO) seeks to infer the parameters of a decision-maker's objective from observed context--action data. We study \emph{noiseless IO}, where demonstrations are generated by a ground-truth objective.
We provide a high-probability $\mathcal{O}(\nicefrac{d}{T})$ generalization bound for the induced action set, where $d$ is the number of unknown parameters and $T$ is the size of the training dataset. We strengthen this guarantee under additional conditions that ensure uniqueness of the chosen action, bringing our IO guarantees in line with best-arm identification results in the bandit literature. 
We further show that the $\mathcal{O}(\nicefrac{d}{T})$ rate is tight over all consistent estimators considered here, and extend the result to both instantaneous and cumulative regret. Notably, the resulting regret lower bound matches the corresponding upper bounds in the adversarial setting, indicating that the stochastic IO setting is effectively adversarial for the class of estimators studied here. Finally, we propose a parameter-free algorithm with lower per-iteration complexity than generic solvers. Experiments validate the predicted rates and illustrate the tightness of our bounds.

\end{abstract}

\section{Introduction}
\label{sec:intro}

We consider the contextual decision-making problem
\begin{equation}\label{eqn:main}
a_{\theta}(s)\in \mathcal{A}_\theta(s)
:=\argmax_{a \in \actionset{s}} F_{\theta}(s,a) ,
\end{equation}
where $s \in \mathcal{S}$ is a random context (state) drawn from a distribution $\mathbb{P}_{\mathcal{S}}$ over $\mathcal{S}$, $\actionset{s}$ is the set of feasible actions given~$s$, and
$F_{\theta}(s,a)= \langle \theta,\psi(s,a)\rangle$ is a linear score parameterized by $\theta \in \mathbb{R}^d$. 
We define $\mathcal{A}_\theta(s)$ as the set of optimal actions for a context $s$ under parameter $\theta$ and assume that it is non-empty. We also define $\psi:\mathcal S\times\mathbb A\to\R^{d}$ as a generic feature map where $\mathbb{A} := \cup_{s\in \mathcal{S}} \; \actionset{s}$. 

Suppose there is an expert who observes a context \(s\) and then selects an action \(a_{\theta^\star}(s)\) according to \eqref{eqn:main} using a parameter \(\theta^\star\) that is unknown to a learner. But the learner gets to observe $T$ demonstrations
\(
D_T := \big\{(s_t,a_t^\star)\big\}_{t=1}^T, \; a_t^\star := a_{\theta^\star}(s_t),
\)
with \((s_t)_{t=1}^T\) drawn i.i.d.\ from \(\mathbb{P}_{\mathcal{S}}\), and wishes to use them to mimic the expert's actions.\footnote{Throughout, we assume that the expert's optimal action is unique, i.e., $\mathcal{A}_{\theta^\star}(s)=\{a_{\theta^\star}(s)\}$ for all $s$. If $\mathcal{A}_{\theta^\star}(s)$ is not a singleton, all of our generalization bounds can be written in the form of $\mathbb{P}_{\mathcal{S}}\big(\mathcal{A}_{\theta^\star}(s)\cap \mathcal{A}_{\hat\theta}(s)=\emptyset\big)$.} 

This learning paradigm is commonly referred to as \emph{inverse optimization} (IO)~\citep{ahuja2001inverse}. The goal of IO is to infer a parameter \(\theta\) such that the corresponding actions generalize to unseen contexts drawn from~$\mathbb{P}_{\cal S}$. Using the demonstrations, the learner outputs an estimate \(\hat\theta\), inducing a \emph{set-valued} greedy map \(\mathcal{A}_{\hat\theta}(s)\). Natural measures to judge a learner's generalization performance are
\begin{subequations}
\label{eq:risk}
\begin{align}
\label{eq:set-risk}
\text{Set-level mismatch:} & \qquad  \mathbb{P}_{\mathcal{S}}\big(a_{\theta^\star}(s)\notin \mathcal{A}_{\hat\theta}(s)\big), \\
\label{eq:action-risk}
\text{Action-level mismatch:} & \qquad \mathbb{P}_{\mathcal{S}}(a_{\theta^\star}(s)\neq a_{\hat\theta}(s)),
\end{align}
\end{subequations}
where \eqref{eq:set-risk} measures how often the learned model excludes the expert's demonstrated action on a fresh state $s\sim\mathbb{P}_{\mathcal{S}}$, while \eqref{eq:action-risk} is a stronger notion, which is also known as \emph{best-arm identification} in the bandit literature~\citep{lattimore2020bandit}. 

IO has been extensively studied from algorithmic, modeling, and optimization perspectives~\citep{chan2025inverse}, and has witnessed a recent surge of interest, e.g., in learning theory \citep{mohajerin2018data,aswani2018inverse,ren2025inverse},  in reinforcement learning \citep{zattoni2025inverse}, for practical problems in control \citep{akhtar2021learning}, robotics \citep{dimanidis2025offline}, routing \citep{zattoni2025inverse}, and finance \citep{li2021inverse}. 

Despite this interest, IO's statistical foundations are far less developed. Since IO can be viewed as a supervised learning problem, it is natural to seek generalization guarantees~\eqref{eq:risk} that characterize its out-of-sample performance. Through the lens of~\eqref{eq:set-risk}, our problem can be seen as a binary classification problem, for which standard learning-theoretic tools could provide generalization guarantees (see \Cref{app:VC_IO}). Similarly, the action-level counterpart~\eqref{eq:action-risk} can be viewed as multiclass or structured prediction, whose PAC complexity is studied via the Natarajan dimension in finite-label settings and characterized by the DS dimension in general multiclass learning~\citep{natarajan1989learning,ben1995characterizations,shalev2014understanding,brukhim2022characterization,pabbaraju2026optimal}. Nonetheless, this viewpoint is limited, as these complexity notions can be hard to quantify and may yield conservative bounds that are often not tight. For instance, existing upper and lower bounds either exhibit gaps in their dependence on the key variables $(d,T)$ or scale with the number of actions, making them uninformative when the action space is infinite~\citep[Theorem~29.3]{shalev2014understanding}.

To the best of our knowledge, this is the first work that addresses {\em tight} generalization guarantees in IO. With this background, let us summarize our key contributions. 

\begin{enumerate}[label=(\roman*), itemsep = 0mm, topsep = 0mm, leftmargin = 7mm]

\item \textbf{Generalization upper bounds.} We develop a \emph{scenario program} view of noiseless IO, where each
demonstration induces a random convex feasibility constraint on the parameter $\theta$ \citep{campi2008exact}. This yields a high-probability $\mathcal{O}(d/T)$ bound on
set-level mismatch (\Cref{thm:scenario-nonexclusion-unique}). We further
upgrade this guarantee to action-level mismatch using 
an incenter estimator that promotes uniqueness of the greedy action
(\Cref{prop:margin-mismatch}). The action-level bound implies instantaneous regret guarantees and, in a stochastic online protocol, logarithmic cumulative regret (\Cref{thm:instantaneous-regret-incenter}).

\item \textbf{Generalization lower bounds.}
We show that the $\mathcal{O}(d/T)$ action-level mismatch rate is minimax optimal
(\Cref{prop:minimax}) and construct a single IO instance whose mismatch
performance~\eqref{eq:risk} attains the upper bound regardless of the choice
of consistent estimator in the considered class
(\Cref{thm:scenario-tightness-continuous}, \Cref{thm:scenario-tightness-action-level}).
We further extend these results to instantaneous and expected cumulative regret,
as the same IO instance applies across datasets of varying sample sizes
(\Cref{thm:scenario-tightness-regret});
\Cref{tab:offline-online-comparison} summarizes these results within the literature.
Notably, the resulting $\Omega(d\log(T/d))$ lower bound on expected cumulative regret matches the cited adversarial upper bounds, even though in our case contexts are sampled i.i.d. For the considered estimator class, our construction exhibits this behavior across a broad class of context distributions. The corresponding lower bound holds with high probability at each fixed round, showing that the per-round difficulty is not driven solely by rare training histories.

\item \textbf{A parameter-free algorithm.}
We propose a parameter-free algorithm that empirically converges to
generalizable solutions after only $T$ iterations. In particular, the proposed
method has lower per-iteration computational cost than off-the-shelf solvers
(\Cref{remark:complexity}).

\end{enumerate}

 Our numerical experiments empirically validate our theoretical predictions and illustrate the tightness of the proposed generalization bounds.

\subsection{Related work}

\textbf{Generalization in IO.}~Classical IO infers objectives and/or constraints that rationalize observed decisions in-sample; guarantees concern exact/approximate optimality on the training instances. A recent out-of-sample analysis by \citet{besbes2025contextual} proves a geometrical upperbound on the minimum possible instantaneous regret a policy could achieve in the adversarial offline setting. Their result depends on the uncertainty angle of the information set but does not provide an explicit rate in the sample size. Our generalization result relies on the set-level and action-level mismatch for a stochastic environment. 

\textbf{Supervised learning.}~From a supervised learning perspective, the IO problem objective in \eqref{eqn:main} can be considered as a hypothesis class for learning the mapping between the context $s$ and the actions $a$. This essentially requires minimizing a loss function to find the model parameter $\theta$.
In IO, the choice of the loss function plays a crucial role. 
Examples of different loss functions include the KKT loss \citep{keshavarz2011imputing}, first order loss \citep{bertsimas2015data}, predictability loss \citep{aswani2018inverse}, suboptimality
loss \citep{mohajerin2018data}, predictive loss, and constrained suboptimality loss \citep{ren2025inverse}.
It is important to note that directly training the objective $F_{\theta^\star}$ is not possible, since we do not have direct measurements of the ground-truth objective values in IO.

\textbf{Smart predict, then optimize.}~In a related study to IO, \citet{elmachtoub2022smart} propose the \textit{Smart predict, then optimize} (SPO) framework, where access to the optimal cost value of $F_{\theta^\star}$ is assumed and the SPO loss function is proposed. Given feature-cost pairs $\{s_t,c_t\}_{t=1}^T$, a predictor is trained such that it minimizes a loss function that penalizes predictions leading to bad decisions. 
SPO assumes access to costs/predictions $c_t:=\langle\theta,s_t\rangle$ and actions/decisions $a^\star_t$, while we only observe states-action pairs~$\{s_t,a_t^\star\}$. In a recent work, \citet{el2023generalization} propose a generalization result of $\mathcal{O}(1/\sqrt{T})$ when assuming a joint distribution over $(s_t,c_t)$. With fewer distributional assumptions but access to more data, they achieve a worse generalization result compared to ours.

\textbf{Online learning.}~Online learning studies the problem of updating a model sequentially as data arrives, with the aim of achieving low cumulative loss or regret \citep{zinkevich2003online,hazan2016introduction}. When performed in an online manner, IO can be viewed as a model-learning task, making it conceptually closer to online learning. 
The major difference between these frameworks is that the available data in online learning is the loss function value, while in IO, it is the optimal decision. Recently, the online version of IO has been studied. \citet{besbes2025contextual} show that in the adversarial case, a naive application of a \textit{circumcenter policy} fails. They also propose an algorithm achieving a regret bound of $\mathcal{O}(d^4\log{T})$. For a similar setting, \citet{gollapudi2021contextual,sakaue2025online} improve this rate to $\mathcal{O}(d\log{T})$ while \citet{sakaue2025online} achieve this rate with improved per-iteration complexity using an online Newton step.
In comparison, we propose a high-probability regret bound of $\mathcal{O}(d\log{T})$ when nature is stochastic using a simple policy. 
This result is interesting since it shows that in the stochastic environment, simple policies can work, while this is not necessarily the case in the adversarial setting.

\begin{table}[t]
\centering
\small
\setlength{\tabcolsep}{3pt}
\renewcommand{\arraystretch}{1.35}

\begin{threeparttable}
\caption{\small Comparison of generalization bounds in IO under adversarial and stochastic settings.}
\label{tab:offline-online-comparison}

\begin{tabular}{
@{}
>{\centering\arraybackslash}m{0.02\linewidth}
>{\centering\arraybackslash}m{0.08\linewidth}
>{\centering\arraybackslash}m{0.15\linewidth}
>{\centering\arraybackslash}m{0.34\linewidth}
>{\centering\arraybackslash}m{0.34\linewidth}
@{}
}
\toprule
&
\textbf{Protocol}
& \textbf{Metric}
& \textbf{Upper bound}
& \textbf{Lower bound} \\
\midrule

\multirow{2}{*}{\rotatebox[origin=c]{90}{\textbf{Adversarial}}}
& Offline
& Instantaneous regret \eqref{eqn:regret_def}
& Geometric characterization; no $T$-dependent rate \citep{besbes2025contextual}
& Matching geometric lower bound \citep{besbes2025contextual} \\

\cmidrule(lr){2-5}

& Online
& Cumulative regret \eqref{eqn:regret_def}
& $\mathcal{O}(d^4 \log T)$ \citep{besbes2025contextual};
  $\mathcal{O}(d \log T)$ \citep{gollapudi2021contextual,sakaue2025online}
& $\Omega(d)$ \citep{sakaue2025online} \\

\midrule

\multirow{2}{*}{\rotatebox[origin=c]{90}{\textbf{Stochastic}}}
& Offline
& Set/Action-level mismatch \eqref{eq:risk}
& \textbf{This work}: $\mathcal{O}\!\left(d/T\right)$ \qquad
  \textup{(\Cref{prop:margin-mismatch})}
& \textbf{This work}: $\Omega\!\left(d/T\right)$
  \textup{(\Cref{prop:minimax}$^\ddagger$ and \Cref{thm:scenario-tightness-continuous}$^{\dagger\ddagger}$)} \\

\cmidrule(lr){2-5}

& Online
& Cumulative regret \eqref{eqn:regret_def}
& \textbf{This work}: $\mathcal{O}\!\left(d \log(T/d)\right)$
  \textup{(\Cref{thm:instantaneous-regret-incenter})}
& \textbf{This work}: $\Omega\!\left(d \log(T/d)\right)$
  \textup{(\Cref{thm:scenario-tightness-regret}$^{\dagger}$)} \\

\bottomrule
\end{tabular}

\begin{tablenotes}[flushleft]
\footnotesize
\item[$\dagger$] The lower bounds are restricted to the algorithm class specified in the corresponding result.
\item[$\ddagger$] The lower bounds differ in their order of quantification (min--max versus max--min); see \Cref{remark:minmax-maxmin}.
\end{tablenotes}
\end{threeparttable}
\end{table}
\subsection{Problem terminology}\label{subsec:prob_term}

For each pair \((s_t,a_t^\star)\) and every \(a\in\actionset{s_t}\), optimality of the expert under
\(\theta^\star\) implies
\(
\langle \theta^\star, \psi(s_t,a)\rangle
\le
\langle \theta^\star, \psi(s_t,a_t^\star)\rangle.
\)
For a generic parameter $\theta\in\mathbb{R}^d$, we say that $\theta$ is
\emph{consistent} with a demonstration $(s_t,a_t^\star)$ if the demonstrated expert action
is also greedy under~$\theta$, i.e.,
\begin{align}
\label{eq:consistency-ineq}
\big\langle \theta,\ \psi(s_t,a)-\psi(s_t,a_t^\star)\big\rangle \le 0,
\ \forall a\in\actionset{s_t},\, t\in[T].
\end{align}
To express condition \eqref{eq:consistency-ineq} compactly, 
define the suboptimality gap \citep{mohajerin2018data}:
\begin{align}
\label{eq:suboptimality-gap}
\ell_\theta^{\mathrm{sub}}(s,a)
:=  \max_{a'\in\actionset{s}} \ \langle \theta,\psi(s,a')-\psi(s,a)\rangle.
\end{align}
By definition, \(\ell_\theta^{\mathrm{sub}}(s,a)\ge 0\) for all $s \in S$ and $a \in \actionset{s}$. Meanwhile, for satisfying condition \eqref{eq:consistency-ineq} we need $\theta$ such that  \(\ell_\theta^{\mathrm{sub}}(s_t,a_t^\star)\le 0\). Combining these two is equivalent to \(\ell_\theta^{\mathrm{sub}}(s_t,a_t^\star)=0\). Therefore, we define the consistency set:
\begin{align}
\label{eq:consistency-set}
\mathcal{C}_T
:=
\Big\{
\theta\in\R^d:\ 
\ell_\theta^{\mathrm{sub}}(s_t,a^\star_{t}) = 0,\ \forall t\in[T] 
\Big\}.
\end{align}
By construction, \(\theta\in\mathcal{C}_T\) if and only if \(a_t^\star\in \mathcal{A}_\theta(s_t)\) for all
\(t\in[T]\), i.e., the demonstrated actions remain greedy under \(\theta\) on the training states.
When \(\mathcal{A}_\theta(s_t)\) contains multiple
actions, the constraint only enforces that the demonstrated action \(a_t^\star\) belongs to that set. If ties occur, $\mathcal{A}_\theta(s)$ may contain multiple actions. 
\begin{remark}[Parameter set regularity]
\label{remark:Theta_scale}
For any \(\theta\in\mathcal{C}_T\), the scaled parameter \(\alpha\theta\) is also in \(\mathcal{C}_T\)
for every \(\alpha>0\). As a result, the consistency set \(\mathcal{C}_T\) forms a cone.
This is expected, since the greedy action is invariant under positive rescaling of \(\theta\).
To avoid this scale ambiguity, we define the closed convex parameter domain \(\Theta\subset \mathbb R^d\) as a normalized set
containing at most one representative from each positive ray: if \(\theta\in\Theta\), then
\(\alpha\theta\notin\Theta\) for all \(\alpha\neq 1\). In particular, this normalization implies
\(0\notin\Theta\). Since positive rescaling does not change the induced actions, we select the normalized representative of $\theta^\star$ and assume, without loss of generality, that \(\theta^\star\in\Theta\).
\end{remark}
The set \(\mathcal{C}_T\) provides an exact characterization of all parameters that rationalize the
observed data. However, this constraint-based description does not by itself yield
out-of-sample guarantees. Each demonstration induces a family of inequalities, and
generalization is governed by the probability that these random feasibility constraints are violated
on unseen states. In the next section, we build on this viewpoint and derive generalization bounds for IO.

\section{Generalization Guarantees}
In this section, we propose high-probability generalization upperbounds for set-level mismatch. Later, we strengthen this guarantee to action-level mismatch as the core result of this work. Through our action-level analysis, we propose regret bounds for a stochastic online variant of our IO framework. 
\subsection{Set-level mismatch upperbound}
Since any IO estimator should belong to $\mathcal{C}_T$, a natural way to select such an estimator is 
\begin{align}
\label{eq:reg-io-scp}
\begin{aligned}
   \hat\theta_T^{\mathrm{sub}} \
 :=   \ &\arg\min_{\theta\in\Theta}\ J(\theta) \quad \text{   s.t.}    \quad \ell_\theta^{\mathrm{sub}}(s_t,a_t^\star) \le 0,  \quad \forall t\in[T],
\end{aligned}
\end{align}
for a convex function $J:\mathbb{R}^d\to\mathbb{R}$ such that the program \eqref{eq:reg-io-scp} admits a \emph{unique} optimizer $\hat\theta_T^{\mathrm{sub}}$ almost surely.
Existence and uniqueness are guaranteed, for example, when $J$ is coercive and strictly convex.

Problem~\eqref{eq:reg-io-scp} reveals a direct connection between IO and the well-studied paradigm of
\emph{scenario optimization}. Each demonstration induces an i.i.d.\ feasibility constraint in parameter
space. 
This connection is significant because it recasts inverse optimization generalization as a constraint-violation problem, focusing on the probability that the learned parameter violates the population feasibility constraint on a fresh state.
To the best of our knowledge, this
scenario-program viewpoint for IO estimation is novel and equips us with new tools.

 We use the result from \citep{campi2008exact} to prove the generalizability of $\hat\theta_T^{\mathrm{sub}}$. 
In particular, the following proposition (see \Cref{app:prf_scenario-nonexclusion-unique} for the proof) states that if $T$ is large enough,  the solution $ \hat\theta_T^{\mathrm{sub}}$ has a low set-level mismatch with high probability.
\begin{proposition}
[Set-level mismatch]
\label{thm:scenario-nonexclusion-unique}
Fix $\beta\in(0,1)$ and choose $\varepsilon\in(0,1]$ such that
\begin{equation*}
T \ge N(\varepsilon,\beta) := \min \biggl\{ n \in \mathbb{N} \;\Big|\; \sum_{i=0}^{d-1} \binom{n}{i}\,\varepsilon^{i}\,(1- \varepsilon)^{n-i} \le \beta \biggr\}.
\end{equation*}
Let $\hat\theta_T^{\mathrm{sub}}$ be the unique optimizer of~\eqref{eq:reg-io-scp}. Then, with probability
at least $1-\beta$ over the draw of $D_T$,
\begin{equation}
\label{eq:nonexclusion-unique}
\mathbb{P}_{\mathcal{S}}\big(a_{\theta^\star}(s)\notin \mathcal{A}_{\hat\theta_T^{\mathrm{sub}}}(s)\big)\ \le\ \varepsilon.
\end{equation}
In particular, if $T\ge 2\big(d+\log(1/\beta)\big)$, then, with probability
at least $1-\beta$ over the draw of $D_T$,
    \begin{equation}
        \label{eq:eps-explicit}
       \mathbb{P}_{\mathcal{S}}\big(a_{\theta^\star}(s)\notin \mathcal{A}_{\hat\theta_T^{\mathrm{sub}}}(s)\big)\ 
        \leq \frac{2}{T}\big(d + \log(1/\beta)\big).
    \end{equation}
\end{proposition}

The guarantee in \Cref{thm:scenario-nonexclusion-unique} is
set-valued, i.e., it only controls whether the expert action is contained in the argmax set
$\mathcal A_{\hat\theta_T^{\mathrm{sub}}}(s)$. When these argmax sets are large, the bound can be
uninformative about the actual action taken by a greedy policy (which depends on tie-breaking).
We illustrate this issue with the following example.

\begin{example}[Issues with set-level guarantees]
\label{example-bigset}
    Let $\mathcal S=\{0,1\}$ with $\mathbb{P}_{\mathcal{S}}(0)=\mathbb{P}_{\mathcal{S}}(1)=\tfrac12$, and let
$\mathbb A=[-1,1]^2$. Define
\(
\psi\big(s,(a_1,a_2)\big):=(a_1,a_2,s\,a_2)\in\mathbb R^3.
\)
For each $s$, the map $a\mapsto \psi(s,a)$ is injective since the first two coordinates
recover $(a_1,a_2)$.
Let $\theta^\star=(1,-1,2)$ and define the expert selection $a_{\theta^\star}(s)\in \mathcal A_{\theta^\star}(s)$.
Since
\(\langle \theta^\star,\psi(s,a)\rangle=a_1+(-1+2s)a_2,\)
the expert argmax sets are singletons and the (unique) maximizers are
$a_{\theta^\star}(0)=(1,-1)$ and $a_{\theta^\star}(1)=(1,1)$.
Now, take $\hat\theta_T^{\mathrm{sub}}=(1,0,0)$. 
This choice of $\hat\theta_T^{\mathrm{sub}}$ is in $\mathcal{C}_T$ since for any given dataset $D_T:=\{(s_t,a_t^\star)\}_{t=1}^T$, any t,  and any $a=(a_1,a_2)\in[-1,1]^2$, the expert always chooses $a_t^\star=(1,\pm 1)$. This means that $(a_t^\star)_1=1$, and therefore \[
  \langle \hat\theta_T^{\mathrm{sub}},\psi(s_t,a)-\psi(s_t,a_t^\star)\rangle
=
a_1-1
\le 0,
\]
where the last inequality is due to $a_1\leq 1$. Therefore, $\hat\theta_T^{\mathrm{sub}}=(1,0,0)\in\mathcal C_T$. For this $\hat\theta_T^{\mathrm{sub}}$, we have
$\langle \hat\theta_T^{\mathrm{sub}},\psi(s,a)\rangle=a_1$ and hence
\[
\mathcal A_{\hat\theta_T^{\mathrm{sub}}}(s)
=
\{(1,a_2):a_2\in[-1,1]\},
\]
 for all $s\in\mathcal S$. This set is a full line segment for every state. In particular,
$a_{\theta^\star}(s)\in A_{\hat\theta_T^{\mathrm{sub}}}(s)$ for both $s=0,1$, and therefore
\(
\mathbb{P}_{\mathcal{S}}\!\big(a_{\theta^\star}(s)\notin \mathcal A_{\hat\theta_T^{\mathrm{sub}}}(s)\big)=0.
\)
\end{example}
To address the issue explained in \Cref{example-bigset}, we need to shift toward action-level mismatch and specify how ties in
$\mathcal A_\theta(s)$ are resolved at test time.
To this end, we seek an estimator of $\theta^\star$ that guarantees recovery of the unique action $a_{\theta^\star}(s)$.
Moreover, controlling the action-level mismatch relates the set-level guarantees to regret bounds. We define instantaneous regret $r_t$ and cumulative regret $R_T$ as
\begin{equation}
\label{eqn:regret_def}
r_t
:=
\big\langle \theta^\star,\psi(s_t,a_t^\star)-\psi(s_t,a_t)\big\rangle,
\quad
R_T:=\sum_{t=1}^T r_t,  
\end{equation}
where $a_t:=a_{\hat\theta_{t-1}}(s_t)$ is the learner's greedy action. \Cref{fig:scenario-tightness-geom}~(a) illustrates how set-level mismatch, action-level mismatch, and the notion of regret, relate to each other.

\begin{figure}[t]
\centering

\begin{minipage}[t]{0.49\textwidth}
\centering
\begin{tikzpicture}[scale=0.98]

\draw[black, very thick, domain=-2.8:3, smooth, variable=\x]
 plot ({\x},{0.39*\x*\x - 1.15});

\draw[mybrown, very thick, dotted, domain=-0.15:5.2, smooth, variable=\x]
 plot ({\x},{0.42*(\x-2.5)^2 -0.15});

\draw[black, very thick] (0,-0.25) ellipse (1.5 and 0.15);
\draw[mybrown, very thick, dashed] (2.5,0.55) ellipse (1.3 and 0.15);

\coordinate (xtheta) at (0,-1.15);
\coordinate (xpoint) at (2.65,-0.15);

\fill[black] (xtheta) circle (3pt);
\fill[mybrown] (xpoint) circle (3pt);

\draw[mybrown, very thick,<->] (0,-1.15)--(0,2.52);
\draw[black, very thick,<->] (2.65,-0.15)--(2.65,1.62);

 \draw[red, very thick,<->] (xpoint)--(xtheta);

\node at (-2.55,2.2) {$-F_{\hat\theta}(s,\cdot)$};
\node[mybrown] at (4.3,2.75) {$-F_{\theta^\star}(s,\cdot)$};

\node[black] at (0,-1.62) {$a_{\hat\theta}(s)$};
\node[mybrown] at (2.67,-0.52) {$a_{\theta^\star}(s)$};

\node[mybrown] at (-1.05,0.9)
{$\underset{\text{Regret}}{\ell_{\theta^\star}^{\mathrm{sub}}\!\bigl(s,a_{\hat\theta}(s)\bigr)}$};

\node[black] at (3.47,1.17)
{\scalebox{0.68}{$\underset{\!\!\!\!\text{Set-level mismatch}}{\ell_{\hat{\theta}}^{\mathrm{sub}}\!\bigl(s,a_{\theta^\star}(s)\bigr)}$}};

\node[red] at (2.25,-1.20)
{$\underset{\text{Action-level mismatch}}{\textcolor{black}{a_{\hat\theta}(s)}\neq \textcolor{mybrown}{a_{\theta^\star}(s)}}$};
\end{tikzpicture}

(a) 
\end{minipage}
\hfill
\begin{minipage}[t]{0.49\textwidth}
\centering

\begin{tikzpicture}[scale=1.5]
\def\L{1.8}
\def\LineExt{1.15}

\pgfmathsetmacro{\aOne}{0.71784}
\pgfmathsetmacro{\bOne}{0.696208}
\pgfmathsetmacro{\cLowOne}{-0.3}
\pgfmathsetmacro{\cHighOne}{1}
\pgfmathsetmacro{\aTwo}{-0.050954}
\pgfmathsetmacro{\bTwo}{0.998701}
\pgfmathsetmacro{\cLowTwo}{-0.3}
\pgfmathsetmacro{\cHighTwo}{1}
\pgfmathsetmacro{\aThree}{-0.829793}
\pgfmathsetmacro{\bThree}{0.558072}
\pgfmathsetmacro{\cLowThree}{-0.3}
\pgfmathsetmacro{\cHighThree}{1}
\pgfmathsetmacro{\aFour}{0.981314}
\pgfmathsetmacro{\bFour}{-0.192415}
\pgfmathsetmacro{\cLowFour}{-0.3}
\pgfmathsetmacro{\cHighFour}{1}
\coordinate (theta) at (0,0);
\coordinate (V1) at (-0.324583,-0.096238);
\coordinate (V2) at (-0.120614,-0.306544);
\coordinate (V3) at (0.165178,-0.291963);
\coordinate (V4) at (0.783934,0.628061);
\coordinate (V5) at (0.402047,1.021813);
\coordinate (V6) at (-0.110483,0.995664);
\coordinate (incenter) at (0.191603,0.462383);
\pgfmathsetmacro{\inrad}{0.399053}
\coordinate (ellCenter) at (0.224384,0.541492);
\pgfmathsetmacro{\ellRx}{0.470406}
\pgfmathsetmacro{\ellRy}{0.41363}
\pgfmathsetmacro{\ellRot}{-112.508158}
\draw[line width=1pt ,densely dotted] (ellCenter) ellipse [x radius=\ellRx, y radius=\ellRy, rotate=\ellRot];

 \fill[black] (incenter) circle (0.8pt);
\node[black, above] at (incenter) {$\hat\theta^{\mathrm{in}}_T \;$};

\draw[->, thick] (-\L,0) -- (\L,0) node[below] {$\theta_{2}$};
\draw[->, thick] (0,-1.3) -- (0,1.7) node[left] {$\theta_{3}$};

\begin{scope}
  \path[fill=red!25, fill opacity=0.5]
    (V1)--(V2)--(V3)--(V4)--(V5)--(V6)--cycle; %
\end{scope}

\DrawStripAsym{\aOne}{\bOne}{\cLowOne}{\cHighOne}{blue!80!black}
\DrawStripAsym{\aTwo}{\bTwo}{\cLowTwo}{\cHighTwo}{green!80!black}
\DrawStripAsym{\aThree}{\bThree}{\cLowThree}{\cHighThree}{brown!70!black}
\DrawStripAsym{\aFour}{\bFour}{\cLowFour}{\cHighFour}{yellow!80!black}

\draw[red!70!black, line width=1.5pt]
  (V1)--(V2)--(V3)--(V4)--(V5)--(V6)--cycle; %

\fill[red!80!black] (V1) circle (0.8pt);
\fill[red!80!black] (V2) circle (0.8pt);
\fill[red!80!black] (V3) circle (0.8pt);
\fill[red!80!black] (V4) circle (0.8pt);
\fill[red!80!black] (V5) circle (0.8pt);
\fill[red!80!black] (V6) circle (0.8pt);

\fill[black] (theta) circle (1pt);
\node[black] at (0.12,-0.12) {$\theta^\star$};
\end{tikzpicture}

(b)
\end{minipage}

\caption{\small (a) Relation between the set-level mismatch, the action-level mismatch, and the true suboptimality loss (regret). Since additive shifts do not affect the optimal actions, we shift the objective curves vertically so that their optimal values coincide. (b) Geometric intuition for tightness in the two-dimensional case $\theta=(1, \theta_2,\theta_3)$ for \Cref{thm:scenario-tightness-continuous}.
Each demonstration induces an asymmetric strip constraint. Intersecting these constraints yields the consistency polytope (shaded) that contains $\theta^\star$.
The incenter estimator $\hat\theta_T^{\mathrm{in}}$ has the maximum distance from the facets of $\mathcal{C}_T$. The dotted ellipsoid represents the intersection of the circular cone around $\hat\theta_T^{\mathrm{in}}$ with $\Theta$. }
\vspace{-0.3cm}
\label{fig:scenario-tightness-geom}
\end{figure} 

\subsection{Action-level mismatch upperbound}

As illustrated in \Cref{example-bigset}, set-level guarantees can be
uninformative when the learned action set
$\mathcal A_{\hat\theta_T^{\mathrm{sub}}}(s)$ is large.
It is possible to avoid such large action sets,
but this requires care in choosing the estimator.
In \Cref{example-bigset}, the estimate $\hat\theta$ lies on the boundary of the consistency set $\mathcal C_T$. With this estimator, the expert's action remains optimal but ties with alternatives at the observed states. Such ties can persist on unseen states. However, this does not rule out suitable solutions in the interior of $\mathcal C_T$. Rather, it motivates selecting an estimator that sits away
from the tie hyperplanes.

In this section, we use an \textit{incenter} estimator inspired by
\citet{zattoni2025learning}, which selects
$\hat\theta_T^{\mathrm{in}}$ in the interior of $\mathcal C_T$. We will show that with high probability, the corresponding action set to the incenter estimator is singleton. Based on this, we improve our generalization bound from set-level to action-level.
Consider the incenter loss function:
\begin{equation}\label{eqn:incenter_loss}
\ell_\theta^{\mathrm{in}}(s,a_{\theta^\star}(s)) :=
\max_{a\in\actionset{s}}
\Big(\langle \theta, \delta(s,a)\rangle + \|\delta(s,a)\|\Big), 
\end{equation}
for \(\delta(s,a) \coloneqq \psi(s,a)-\psi\big(s,a_{\theta^\star}(s)\big)\).
For fixed $s$, $\theta\mapsto \ell_\theta^{\mathrm{in}}(s,a_{\theta^\star}(s))$ is convex as a pointwise supremum of convex
functions. 
Using $\ell_\theta^{\mathrm{in}} (s_t,a_{t}^\star)$ instead of the sub-optimality loss in \eqref{eq:reg-io-scp} would result in the following unique incenter estimator:
\begin{align}
\label{eq:scp-incenter}
\begin{aligned}
   \hat\theta_T^{\mathrm{in}} \
 :=   \ &\arg\min_{\theta\in\R^d}\ J(\theta)  \quad \text{   s.t.}    \quad \ell_\theta^{\mathrm{in}}(s_t,a_t^\star) \le 0,  \quad \forall t\in[T].
\end{aligned}
\end{align}
If $\operatorname{int}(\mathcal C_T)\neq\emptyset$, then
\eqref{eq:scp-incenter} is feasible over $\mathbb R^d$
\citep{zattoni2025learning}.
Moreover, the constraint $\ell_\theta^{\mathrm{in}}(s_t,a_t^\star)\le0$
is equivalent to
\begin{equation}
\label{eq:incenter_margin_constraints}
\langle\theta,\delta(s_t,a)\rangle
\le-\|\delta(s_t,a)\|,
\qquad \forall a\in\actionset{s_t}.
\end{equation}
Under our standing assumption that the expert's optimal action is unique,
every alternative $a\neq a_t^\star$ satisfies $\delta(s_t,a)\neq0$.
Thus, \eqref{eq:incenter_margin_constraints} enforces a strictly negative
margin separating the demonstrated action from every alternative.
In contrast, the consistency constraints only require
$\langle\theta,\delta(s_t,a)\rangle\le0$, permitting boundary solutions
that can induce ties. Thus, the incenter loss keeps the estimator
away from the tie hyperplanes $\langle\theta,\delta(s_t,a)\rangle=0$.
Next, we establish the generalization guarantee for
$\hat\theta_T^{\mathrm{in}}$. The proof is given in \Cref{app:prf_margin-mismatch}.
\begin{theorem}[Action-level guarantees via incenter]
\label{prop:margin-mismatch}
Assume
$\mathrm{int}(\mathcal{C}_T)\neq\emptyset$ almost surely.
Fix $\beta\in(0,1)$ and let $\varepsilon\in(0,1]$ satisfy $T\ge N(\varepsilon,\beta)$.
Then, with probability at least $1-\beta$ over the draw of $D_T$,
$\hat\theta_T^{\mathrm{in}}$ satisfies
\begin{equation}
\mathbb P_{\mathcal S}\!\left(
\mathcal A_{\hat\theta_T^{\mathrm{in}}}(s)
=
\{a_{\theta^\star}(s)\}
\right)
\ge 1-\varepsilon.
\end{equation}
\end{theorem}
\Cref{prop:margin-mismatch} is our main generalization result.
It strengthens the set-level guarantee by ensuring that the
learned objective has a unique optimal action matching the expert,
with the same sample complexity $N(\varepsilon,\beta)$.
The only additional structural requirement is
$\operatorname{int}(\mathcal C_T)\neq\emptyset$ almost surely. \Cref{remark:finite-action} proposes a sufficient condition for this requirement. 

\begin{remark}[Finite action sets \& PAC learnability]
\label{remark:finite-action}
For finite feasible action sets $\mathbb A(s_t)$, the condition
$\operatorname{int}(\mathcal C_T)\neq\emptyset$ holds automatically
under our assumption that the expert's optimal action is unique.
Indeed, at every observed context, each alternative
$a\in\mathbb A(s_t)\setminus\{a_t^\star\}$ satisfies
${\langle\theta^\star,\delta(s_t,a)\rangle<0}$.
Since there are only finitely many such strict linear inequalities,
they remain strict throughout a sufficiently small neighborhood
of $\theta^\star$. This neighborhood is therefore contained in
$\mathcal C_T$, proving
$\theta^\star\in\operatorname{int}(\mathcal C_T)$.
This observation connects our action-level guarantee to PAC learnability
for linear structured prediction with finite action sets
\citep[Theorem~5]{daniely2014optimal}.
In this setting, our sample-complexity bound in \Cref{prop:margin-mismatch} improves their upper bound
by removing the additional $\log(1/\varepsilon)$ factor multiplying $d$.
Furthermore, our result extends their finite-action guarantee to infinite
feasible action sets $\mathbb A(s_t)$ under the assumption that
$\operatorname{int}(\mathcal C_T)\neq\emptyset$ almost surely.
\end{remark}

\textbf{Revisiting \Cref{example-bigset} for incenter solution.}
Under the incenter formulation~\eqref{eq:scp-incenter},
the parameter $\hat\theta_T^{\mathrm{sub}}=(1,0,0)$ is no longer feasible.
Indeed, for $s=0$ and the competitor $a'=(1,1)$, we have $a_{\theta^\star}(0)=(1,-1)$ and $\delta(0,a')=(0,2,0)$, and therefore
\(
\langle (1,0,0),\delta(0,a')\rangle+\|\delta(0,a')\| = 2 > 0,
\)
violating the incenter constraint.
A feasible incenter solution is, for instance,  
\(
\hat\theta_T^{\mathrm{in}}=(2,-2,6),
\)
which satisfies all incenter constraints and yields the greedy score
\[
\langle \hat\theta_T^{\mathrm{in}},\psi(s,a)\rangle
=
2a_1+(-2+6s)a_2.
\]
As a result, the induced argmax sets are singletons:
\(
\mathcal A_{\hat\theta_T^{\mathrm{in}}}(0)=\{(1,-1)\},
\mathcal A_{\hat\theta_T^{\mathrm{in}}}(1)=\{(1,1)\}.
\)
Hence, the incenter estimator yields a unique action set. 

\subsection{Regret upperbound}

We extend our IO framework to a stochastic online protocol.
At each round $t=1,2,\dots,T$, a fresh context $s_t\sim\mathbb{P}_{\mathcal{S}}$ arrives; the learner
(i) computes an IO estimate from the demonstrations observed so far, (ii) selects a greedy action under
this estimate for the new context, and (iii) observes the expert action under the same context. For the regret analysis, we impose the following boundedness assumption on the feature map.
\begin{assumption}[Bounded features]\label{ass:bounded-features}
There exists $C > 0$ such that
\(
\forall (s,a)\in\mathcal{S}\times\mathbb{A}, \|\psi(s,a)\|_2 \le C.
\)
\end{assumption}
This assumption ensures that each action mismatch incurs bounded regret, allowing us to translate action-level mismatch guarantees into regret bounds. Without such control, even a small mismatch probability can lead to arbitrarily large regret.
Recall that
\(
\delta(s,a)=\psi(s,a)-\psi\big(s,a_{\theta^\star}(s)\big).
\)
Then, by \Cref{ass:bounded-features},
\(
\|\delta(s,a)\|_2\le B:=2C.
\)
We evaluate performance via regret with respect to the expert score under $\theta^\star$ (see definitions of instantaneous regret $r_t$ and cumulative regret $R_T$ in \eqref{eqn:regret_def}).

\begin{proposition}[Regret upper bound]
\label{thm:instantaneous-regret-incenter}
Suppose \Cref{ass:bounded-features} holds and
$\operatorname{int}(\mathcal C_t)\neq\emptyset$ almost surely
for every $t\in\{1,\ldots,T-1\}$. Consider the
stochastic online protocol based on the incenter estimator, where
$a_t=a_{\hat\theta^{\mathrm{in}}_{t-1}}(s_t)$. Fix $\beta\in(0,1)$ and let
$\varepsilon\in(0,1]$ satisfy $t-1\ge N(\varepsilon,\beta)$. Then, with
probability at least $1-\beta$ over the draw of $D_{t-1}$,
\begin{equation}
    \mathbb{E}_{\mathcal{S}}[r_t\mid D_{t-1}]
    \le
    \|\theta^\star\|_2\,B\,\varepsilon .
\end{equation}
Moreover, for every $\delta\in(0,1)$ and $T\ge d+1$, with probability at least
$1-\delta$ over $D_T$,
\begin{equation}
R_T
=
\mathcal{O}\!\left(
\|\theta^\star\|_2\,B\,
\Big(d+\log\frac{T}{\delta}\Big)
\log \frac{T}{d}
\right).
\end{equation}
\end{proposition}

A more complete statement and proof of \Cref{thm:instantaneous-regret-incenter} are given in \Cref{app:online-proof-main}.
Our resulting $\mathcal{O}(d\log T)$-type regret matches the guarantees obtained by
cutting-plane approaches~\citep{gollapudi2021contextual}
 and second-order online updates~\citep{sakaue2025online}.
A key conceptual distinction of this work with \citep{besbes2025contextual} is the role of the environment. In particular, their analysis shows that a naive greedy circumcenter policy can suffer linear regret under adversarial instances,
whereas in our stochastic setting the greedy incenter protocol enjoys logarithmic regret.

\section{Tightness of the Guarantees}
\label{sec:tightness}
We establish two complementary tightness results. First, a minimax
lower bound shows that the action-level generalization rate of
$\hat\theta_T^{\mathrm{in}}$ is optimal
up to constants among all learning algorithms. Second, we construct
a single IO instance that attains the set-level bound in
\Cref{thm:scenario-nonexclusion-unique} exactly for every estimator
in a restricted class and every $T\ge d$. This common-instance
construction also establishes instantaneous and expected cumulative
regret lower bounds for that class.
\label{sec:tightness}
\subsection{Minimax lower bound}

We first show that the action-level rate in
\Cref{prop:margin-mismatch} cannot be improved in general. This result follows by connecting realizable binary classification to our IO framework and applying the lower bound established by
\citet{ehrenfeucht1989general}.

\begin{proposition}[Minimax lower bound]
\label{prop:minimax}
Let $\mathcal A_T$ be any algorithm mapping $D_T$ to a measurable policy $\widehat a_T=\mathcal A_T(D_T)$. There exist constants $c,\varepsilon_0,\delta_0>0$ such that, for every $\varepsilon\in(0,\varepsilon_0]$ and $\delta\in(0,\delta_0]$, if
$
T
\le
c\frac{d+\log(1/\delta)}{\varepsilon},
$
then
$$
\inf_{\mathcal A_T}
\sup_{\theta^\star,P_S}
\mathbb P_{D_T,\mathcal A_T}
\left(
P_S\left(
\widehat a_T(s)\neq a_{\theta^\star}(s)
\right)>\varepsilon
\right)
\ge\delta.
$$
\end{proposition}
The proof of  \Cref{prop:minimax} is given in \Cref{proof:minimax}.  The lower bound applies to every learning algorithm, in particular to the incenter estimator. Consequently, the action-level rate
in \Cref{prop:margin-mismatch} is minimax optimal up to universal constants.

\Cref{prop:minimax} holds for every algorithm, but the hard instance $(\theta^\star,P_S)$ may depend on the algorithm
and the sample size. The following subsections ask whether a single IO instance
can be hard for every estimator in a specified class and every $T\ge d$.
 We establish a uniform-instance lower bound
for scenario-program estimators induced by the following restricted
objective class in \eqref{eq:reg-io-scp}:
\[
\mathcal{J}
:=
\left\{
\theta\mapsto f\!\left(\left\langle c,\theta\right\rangle+b\right)
\;\middle|\;
f:\R\to\R \text{ is strictly monotone},\
b\in\R,\ 
\langle c,\theta\rangle \text{ is non-constant on } \Theta
\right\}.
\]

\subsection{Set-level and action-level mismatch lowerbounds}
 The bound from \citep[Theorem 1]{campi2008exact} provides a tightness guarantee when the scenario program satisfies a fully-supported condition (see \Cref{def:fully-supp} for the definition of a fully-supported problem). We build an IO instance corresponding to a fully-supported scenario program, which shows that the bound from \Cref{thm:scenario-nonexclusion-unique} is not improvable in general.

\begin{theorem}[Tightness of set-level mismatch]
\label{thm:scenario-tightness-continuous}
For every $T\ge d$ and $\varepsilon\in(0,1]$,
$\hat\theta_T^{\mathrm{sub}}$ satisfies
\begin{equation}
\sup_{\theta^\star, \mathbb{P}_{\mathcal{S}}} \inf_{J \in \mathcal{J}} \; 
\mathbb{P}_D\!\left(
\mathbb{P}_{\mathcal{S}}\big(a_{\theta^\star}(s)\notin
\mathcal{A}_{\hat\theta_T^{\mathrm{sub}}}(s)\big) > \varepsilon
\right)
\ge
\beta
:=
\sum_{i=0}^{d-1}\binom{T}{i}\varepsilon^{i}(1-\varepsilon)^{T-i}.
\end{equation}
In particular, if $\varepsilon$ is chosen so that the
right-hand side equals $\beta$, then, up to constants,
\(
\varepsilon
\asymp
\frac{d+\log(1/\beta)}{T}.
\)
\end{theorem}

\begin{proof}
Let $\mathcal S=\mathbb S^{d-1}$, and let
$s\sim\mathbb{P}_{\mathcal{S}}$ have a density on the sphere, i.e.,
$\mathbb{P}_{\mathcal{S}}$ is absolutely continuous with respect to the
uniform surface-area measure on $\mathbb S^{d-1}$. Let $\mathbb A=[-1,1]$.
Parametrize $\theta\in\R^{d+1}$ as $\theta=(\theta_1,\theta_{-1})$ and restrict
to the affine slice $\Theta:=\{(1,\theta_{-1}):\theta_{-1}\in\R^d\}$.
Define $\psi$ by
\[
\psi(s,a):=\binom{-2|a|+a}{a\,s}.
\]
Then $F_\theta(s,a)=-2|a|+a+a\,s^\top\theta_{-1}$. Let
$\theta^\star=(1,\bm 0)$. Since $F_{\theta^\star}(s,a)=-2|a|+a\le 0$, with
equality only at $a=0$, we have $a_{\theta^\star}(s)=0$ for all
$s\in\mathcal S$. For the action $0$, we have $F_\theta(s,0)=0$,
and hence $\ell_\theta^{\mathrm{sub}}(s,0)=\max_{a\in[-1,1]}F_\theta(s,a)$.
Splitting over $a\in[0,1]$ and $a\in[-1,0]$ gives
\[
\ell_\theta^{\mathrm{sub}}(s,0)
=
\max\{0,\ s^\top\theta_{-1}-1,\ -s^\top\theta_{-1}-3\}.
\]
Thus $\ell_\theta^{\mathrm{sub}}(s,0)\le 0$ is equivalent to the asymmetric
strip constraints $s^\top\theta_{-1}\le 1$ and $-s^\top\theta_{-1}\le 3$.
For any objective $J(\theta)=f(\langle c,\theta\rangle+b)\in\mathcal J$,
strict monotonicity of $f$ and non-constancy of $\langle c,\theta\rangle$ on
$\Theta$ imply that minimizing $J$ over the feasible polytope is equivalent to
either minimizing or maximizing a non-constant linear functional in
$\theta_{-1}$. Hence the scenario program~\eqref{eq:reg-io-scp} reduces to an
LP of the form
\begin{equation}
\label{eqn:scp_tight_proof}
\hat\theta_{-1}\in
\arg\min_{\theta_{-1}\in\R^d}
\ c_{-1}^\top\theta_{-1}
\quad \text{s.t.}\quad
s_t^\top\theta_{-1}\le 1,\quad
-s_t^\top\theta_{-1}\le 3,\quad t\in[T].
\nonumber
\end{equation}
Under the absolute continuity of $\mathbb{P}_{\mathcal{S}}$, for $T\ge d$ the
directions $(s_t)_{t=1}^T$ span $\R^d$ almost surely, so the feasible set is a
bounded polytope. Moreover, the linear objective yields a unique vertex of the
polytope almost surely. At such an optimum, exactly $d$ constraints are active
almost surely. Since the two inequalities generated by the same sample cannot
be simultaneously active, these active constraints correspond to $d$ distinct
samples. As a result, the scenario program is fully supported with support size
$d$ almost surely.
Hence, by \citep[Theorem~1]{campi2008exact}, the violation probability of the
scenario optimizer obeys the exact tail identity. Since this violation
probability is precisely
\(
\mathbb{P}_{\mathcal{S}}\big(a_{\theta^\star}(s)\notin
\mathcal{A}_{\hat\theta_T^{\mathrm{sub}}}(s)\big),
\)
taking the infimum over $J\in\mathcal J$ and then the supremum over
$(\theta^\star,\mathbb{P}_{\mathcal S})$ gives the stated lower bound.

In the simplest nontrivial case $\theta_{-1}\in\R^2$, each sample induces a
random strip constraint, so the consistency set becomes a random polytope
around $\theta^\star$. Minimizing $J$ selects a vertex determined by $d$ active
constraints; see \Cref{fig:scenario-tightness-geom} (b).
\end{proof}

\Cref{thm:scenario-tightness-continuous} shows that
the dependence on $(d,T,\beta)$ in \Cref{thm:scenario-nonexclusion-unique} cannot be improved in a
distribution-free sense and without assuming additional structure.

\begin{remark}[Min--max vs. max--min lower bounds]
\label{remark:minmax-maxmin}
Let $L(\mathcal A,I)$ denote the mismatch probability of algorithm $\mathcal A$ on instance $I$.
For any fixed algorithm class,
\[
\sup_I\inf_{\mathcal A}L(\mathcal A,I)
\le
\inf_{\mathcal A}\sup_I L(\mathcal A,I).
\]
Thus, a max--min lower bound is stronger than the corresponding minimax lower bound for the same class: it requires a common hard instance rather than allowing the instance to depend on the algorithm.
\Cref{prop:minimax} establishes a minimax lower bound over all learners, whereas \Cref{thm:scenario-tightness-continuous} gives this stronger conclusion within the restricted class $\mathcal J$, constructing a single instance that is hard for every estimator induced by $\mathcal J$.
Specifying this class is essential: for each fixed instance, an inner infimum over arbitrary convex objectives could select
\(
J(\theta)=\|\theta-\theta^\star\|_2^2,
\)
and the resulting estimator uniquely recovers $\theta^\star$ and has zero mismatch.
Restricting the objective to $\mathcal J$ rules out this degeneracy while retaining the structured class of boundary-selecting estimators considered in \Cref{thm:scenario-tightness-continuous}.
\end{remark}

\begin{remark}[Tightness of action-level mismatch]
\label{thm:scenario-tightness-action-level}
     For any
tie-breaking rule
$a_{\hat\theta_T^{\mathrm{sub}}}(s)\in
\mathcal A_{\hat\theta_T^{\mathrm{sub}}}(s)$, if
\(
a_{\theta^\star}(s)\notin
\mathcal A_{\hat\theta_T^{\mathrm{sub}}}(s)
\)
then we have
\(
a_{\theta^\star}(s)\neq a_{\hat\theta_T^{\mathrm{sub}}}(s).\) Hence, the lower bound on set-level mismatch in \Cref{thm:scenario-tightness-continuous} immediately
transfers to action-level mismatch.
\end{remark}

\subsection{Regret lower bound}

We next convert the action-level lower bound in \Cref{thm:scenario-tightness-action-level}
into a regret lower bound. In the proof of 
\Cref{thm:scenario-tightness-continuous}, every action-level error incurs constant positive regret.
Thus, the instantaneous regret inherits the same lower tail, and summing over
time yields a logarithmic cumulative regret lower bound.
The proof of the following result is given in \Cref{proof:scenario-tightness-regret}.
\begin{proposition}[Tightness of regret]
\label{thm:scenario-tightness-regret}
Consider a protocol that at each round $t$ selects
$a_t\in\mathcal A_{\hat\theta_{t-1}^{\mathrm{sub}}}(s_t)$ according to any measurable tie-breaking rule. Then, for every $t\ge d+1$ and
$\varepsilon\in(0,1]$,
\begin{equation}
\sup_{\theta^\star,\mathbb{P}_{\mathcal S}}
\inf_{J\in\mathcal J}
\mathbb{P}_{D_{t-1}}\!\left(
\mathbb{E}_{\mathcal{S}}[r_t\mid D_{t-1}]\ge \varepsilon
\right)
\ge
\sum_{i=0}^{d-1}\binom{t-1}{i}
\varepsilon^i(1-\varepsilon)^{t-1-i}.
\end{equation}
In particular, for every $T\ge d+1$, we have
\begin{equation}
\sup_{\theta^\star,\mathbb{P}_{\mathcal S}}
\inf_{J\in\mathcal J}
\mathbb{E}[R_T]
\ge
d\log\frac{T+1}{d+1}.
\end{equation}
\end{proposition}
As shown in \Cref{tab:offline-online-comparison}, this
$\Omega(d\log T)$ lower bound in \Cref{thm:scenario-tightness-regret} matches the
known adversarial upper bounds. Thus, within the considered estimator class,
the stochastic IO setting is effectively not easier than the adversarial one.

\section{A Parameter-free Algorithm}
\label{sec:param_free}
Our analysis has mainly focused on statistical guarantees for exact solutions returned by either
$\hat\theta_T^{\mathrm{sub}}$ in \eqref{eq:reg-io-scp} or $\hat\theta_T^{\mathrm{in}}$ in~\eqref{eq:scp-incenter}. These programs can be solved using standard convex solvers. Recall that
the auxiliary objective $J(\theta)$ is used to ensure uniqueness of the optimizer. Since $J(\theta)$
is generic, we expect that any optimization method that produces an approximate solution 
in $\mathcal C_T$ exhibits similar generalization behavior in practice. Define
\[
    f_T(\theta)
    :=
    \max_{1\le t\le T}
    \ell_\theta^{\mathrm{sub}}\big(s_t,a_t^\star\big)
    =
    \max_{1\le t\le T}\max_{a\in\actionset{s_t}}\ \langle \theta, \delta(s_t,a)\rangle.
\]
By definition of the consistency set $\mathcal C_T$,
finding a point in $\mathcal C_T$ is equivalent to minimizing $f_T(\theta)$,
whose optimal value is zero. Since $f_T$ is a non-smooth convex function, we can minimize it with a
subgradient method using the Polyak step-size, yielding a parameter-free procedure. Let
$\theta_i^{\mathrm{Pol}}$ denote the iterate at iteration $i$. Choose the most violated constraint $t_i$ and its corresponding action $a_i$ at
$\theta_i^{\mathrm{Pol}}$ as the maximizer of $\langle \theta_i^{\mathrm{Pol}}, \delta(s_t,a)\rangle$ for $1\leq t\leq T$ and $a\in\actionset{s_t}$.
Then $\delta(s_{t_i},a_i)$ is a subgradient of $f_T$ at $\theta_i^{\mathrm{Pol}}$, and the Polyak
update is
\begin{align}
\label{eqn:polyak_alg}
   \theta_{i+1}^{\mathrm{Pol}}
   \leftarrow
   \theta_i^{\mathrm{Pol}}
   - \frac{f_T(\theta_i^{\mathrm{Pol}})}{\|\delta(s_{t_i},a_i)\|_2^2}\,\delta(s_{t_i},a_i).
\end{align}
Standard results for Polyak step-sizes on non-smooth convex objectives yield a convergence rate of
$\mathcal O(1/\sqrt{N})$ after $N$ iterations of \eqref{eqn:polyak_alg}. In particular, for a finite
number of iterations we cannot guarantee that $\theta_N^{\mathrm{Pol}}\in\mathcal C_T$, and therefore
the generalization bound of \Cref{thm:scenario-nonexclusion-unique} does not directly apply to
$\theta_N^{\mathrm{Pol}}$. Nevertheless, as we show empirically in \Cref{sec:numerics},
$\theta_N^{\mathrm{Pol}}$ exhibits generalization behavior similar to $\hat\theta_T^{\mathrm{sub}}$
when $N=T$, i.e., when we perform $T$ Polyak updates. 
This observation suggests that extending our
theory to approximate solutions is an interesting direction for future work. Finally, to avoid the
trivial $\theta=0$ solution, one should initialize \eqref{eqn:polyak_alg} away from the origin.

\begin{remark}[Complexity reduction]\label{remark:complexity}
The algorithm \eqref{eqn:polyak_alg} has lower per-iteration cost than generic
solvers (e.g., interior-point methods) for \eqref{eq:reg-io-scp}. When $\mathbb A$ is finite, each update \eqref{eqn:polyak_alg} identifies the most violated pair $(t_i,a_i)$ via a scan over
$t\in[T]$ and $a\in\mathbb A$, costing $\mathcal O(T|\mathbb A|)$. In contrast, an interior-point
method must handle $ T|\mathbb A|$ inequality constraints, and its worst-case
per-iteration cost is $\mathcal O(T^3|\mathbb A|^3)$.
\end{remark}

\begin{figure*}[t]
\centering
\begin{tikzpicture}

\begin{groupplot}[
  group style={
    group size=3 by 1,
    vertical sep=1.2cm,
  },
  width=0.36\linewidth,
  height=0.35\linewidth,
  xmode=log,
  ymode=log,
  xmin=2, xmax=316,
  ymin=7e-3, ymax=1.2,
  grid=both,
  xtick={2,3,4,5,6,7,8,9,10,11,21,31,41,51,61,71,81,91,101,201,300},
  xticklabels={$10^0$, , , , , , , , ,$10^1$, , , , , , , , ,$10^2$, , },
  xlabel style = {font=\small},
  ylabel style = {font=\small},
  title style = {font=\small},
  legend style={
    at={(2.59,0.02)},
    anchor=south west,
    draw=none,
    fill=white,
    fill opacity=1,
    text opacity=1,
    font=\scriptsize,
  },
  legend image code/.code={
    \draw[#1] (0cm,0cm) -- (0.8cm,0cm);
  },
  unbounded coords=jump,
]

\nextgroupplot[
ylabel={Mismatch probability},
xlabel={\# Samples $T$},
  title={Dimension $d=5$},
  legend entries={
    Upper Bound,
    ${\mathrm{GE}}_T^{\text{sub}}$,
    ${\mathrm{GE}}_T^{\text{in}}$,
    ${\mathrm{GE}}_T^{\text{Pol}}$
  },
  legend style={
    font=\tiny,          
    cells={anchor=west},
  },
]

\addplot[black, dashed, line width=0.8pt]
table[x=T, y=epsilon, col sep=comma] {./theory_d5.csv};

\addplot[blue, line width=0.4pt]
table[x=T, y=avg_gen_prob, col sep=comma] {./gen_df5_1000.csv};
\addplot[ForestGreen!80!black, line width=0.4pt]
table[x=T, y=avg_gen_prob, col sep=comma] {./gen_df_inc5_1000.csv};
\addplot[orange, line width=0.4pt]
table[x=T, y=avg_gen_prob, col sep=comma] {./gen_df_on5_1000.csv};

\addplot[name path=df3_base_u, draw=none]
table[x=T, y=ci90_upper, col sep=comma] {./gen_df5_1000.csv};
\addplot[name path=df3_base_l, draw=none]
table[x=T, y=ci90_lower, col sep=comma] {./gen_df5_1000.csv};
\addplot[blue, draw=none, fill opacity=0.20, forget plot]
fill between[of=df3_base_u and df3_base_l];

\addplot[name path=df3_inc_u, draw=none]
table[x=T, y=ci90_upper, col sep=comma] {./gen_df_inc5_1000.csv};
\addplot[name path=df3_inc_l, draw=none]
table[x=T, y=ci90_lower, col sep=comma] {./gen_df_inc5_1000.csv};
\addplot[ForestGreen!80!black, draw=none, fill opacity=0.20, forget plot]
fill between[of=df3_inc_u and df3_inc_l];

\addplot[name path=df3_pol_u, draw=none]
table[x=T, y=ci90_upper, col sep=comma] {./gen_df_on5_1000.csv};
\addplot[name path=df3_pol_l, draw=none]
table[x=T, y=ci90_lower, col sep=comma] {./gen_df_on5_1000.csv};
\addplot[orange, draw=none, fill opacity=0.20, forget plot]
fill between[of=df3_pol_u and df3_pol_l];

\nextgroupplot[
xlabel={\# Samples $T$},
  title={Dimension $d=10$}
]
\addplot[blue, line width=0.4pt]
table[x=T, y=avg_gen_prob, col sep=comma] {./gen_df10_1000.csv};
\addplot[ForestGreen!80!black, line width=0.4pt]
table[x=T, y=avg_gen_prob, col sep=comma] {./gen_df_inc10_1000.csv};
\addplot[orange, line width=0.4pt]
table[x=T, y=avg_gen_prob, col sep=comma] {./gen_df_on10_1000.csv};

\addplot[black, dashed, line width=0.8pt]
table[x=T, y=epsilon, col sep=comma] {./theory_d10.csv};

\addplot[name path=df10_base_u, draw=none]
table[x=T, y=ci90_upper, col sep=comma] {./gen_df10_1000.csv};
\addplot[name path=df10_base_l, draw=none]
table[x=T, y=ci90_lower, col sep=comma] {./gen_df10_1000.csv};
\addplot[blue, draw=none, fill opacity=0.20, forget plot]
fill between[of=df10_base_u and df10_base_l];

\addplot[name path=df10_inc_u, draw=none]
table[x=T, y=ci90_upper, col sep=comma] {./gen_df_inc10_1000.csv};
\addplot[name path=df10_inc_l, draw=none]
table[x=T, y=ci90_lower, col sep=comma] {./gen_df_inc10_1000.csv};
\addplot[ForestGreen!80!black, draw=none, fill opacity=0.20, forget plot]
fill between[of=df10_inc_u and df10_inc_l];

\addplot[name path=df10_pol_u, draw=none]
table[x=T, y=ci90_upper, col sep=comma] {./gen_df_on10_1000.csv};
\addplot[name path=df10_pol_l, draw=none]
table[x=T, y=ci90_lower, col sep=comma] {./gen_df_on10_1000.csv};
\addplot[orange, draw=none, fill opacity=0.20, forget plot]
fill between[of=df10_pol_u and df10_pol_l];

\nextgroupplot[
xlabel={\# Samples $T$},
 title={Dimension $d=20$},
]
\addplot[blue, line width=0.4pt]
table[x=T, y=avg_gen_prob, col sep=comma] {./gen_df20_1000.csv};
\addplot[ForestGreen!80!black, line width=0.4pt]
table[x=T, y=avg_gen_prob, col sep=comma] {./gen_df_inc20_1000.csv};
\addplot[orange, line width=0.4pt]
table[x=T, y=avg_gen_prob, col sep=comma] {./gen_df_on20_1000.csv};

\addplot[black, dashed, line width=0.8pt]
table[x=T, y=epsilon, col sep=comma] {./theory_d20.csv};

\addplot[name path=df30_base_u, draw=none]
table[x=T, y=ci90_upper, col sep=comma] {./gen_df20_1000.csv};
\addplot[name path=df30_base_l, draw=none]
table[x=T, y=ci90_lower, col sep=comma] {./gen_df20_1000.csv};
\addplot[blue, draw=none, fill opacity=0.20, forget plot]
fill between[of=df30_base_u and df30_base_l];

\addplot[name path=df30_inc_u, draw=none]
table[x=T, y=ci90_upper, col sep=comma] {./gen_df_inc20_1000.csv};
\addplot[name path=df30_inc_l, draw=none]
table[x=T, y=ci90_lower, col sep=comma] {./gen_df_inc20_1000.csv};
\addplot[ForestGreen!80!black, draw=none, fill opacity=0.20, forget plot]
fill between[of=df30_inc_u and df30_inc_l];

\addplot[name path=df30_pol_u, draw=none]
table[x=T, y=ci90_upper, col sep=comma] {./gen_df_on20_1000.csv};
\addplot[name path=df30_pol_l, draw=none]
table[x=T, y=ci90_lower, col sep=comma] {./gen_df_on20_1000.csv};
\addplot[orange, draw=none, fill opacity=0.20, forget plot]
fill between[of=df30_pol_u and df30_pol_l];

\end{groupplot}

\end{tikzpicture}
\caption{\small Generalization mismatch probability $\mathrm{GE}_T$ on fresh contexts versus $T$ for
$d\in\{5,10,20\}$. The depicted results are averaged over $10$ runs with $90\%$ confidence bands. The dashed line shows the
theoretical upper bound from \Cref{thm:scenario-nonexclusion-unique} with $\beta=0.1$. The experiments confirm our theoretical findings.}
\vspace{-0.3cm}

\label{fig:numerics_group}
\end{figure*}
 
\section{Numerical Experiments}\label{sec:numerics}
In this section, we verify the generalization guarantees in \Cref{thm:scenario-nonexclusion-unique},
\Cref{prop:margin-mismatch}, and the parameter-free method of \Cref{sec:param_free} on synthetic data. The details of our experimental setup are given in \Cref{app:numerical}.
We compare three estimators $\hat\theta_T^{\mathrm{sub}}$, $\hat\theta_T^{\mathrm{in}}$, and
$\hat\theta_T^{\mathrm{Pol}}$ derived from \eqref{eq:reg-io-scp}, \eqref{eq:scp-incenter}, and the
parameter-free approach in \Cref{sec:param_free}, respectively.
\Cref{fig:numerics_group} shows the generalization mismatch probability on fresh contexts for
$d\in\{5,10,20\}$. In all three panels, the mismatch drops as we collect more data, and the
curves follow an $\mathcal O(1/T)$ decay, matching the predictions of
\Cref{thm:scenario-nonexclusion-unique} and \Cref{prop:margin-mismatch}. The incenter estimator
$ \hat{\theta}_T^{\mathrm{in}}$ performs best across all dimensions, suggesting that enforcing
a strict margin reduces ties in the greedy rule and improves action prediction, consistent with
\Cref{prop:margin-mismatch}. The parameter-free Polyak method $ \hat{\theta}_T^{\mathrm{Pol}}$ 
tracks $ \hat{\theta}_T^{\mathrm{sub}}$ closely, providing a solver-free alternative with similar statistical performance. Comparing the three panels,
the mismatch is larger for higher dimensions, which agrees with the linear dependence on $d$
suggested by our theory. 

\section{Limitations and Concluding Remarks}

We established tight generalization guarantees for IO in the noiseless regime by viewing empirical loss as a scenario program for the first time. Our work comes with its own limitations that naturally suggest future research directions.

{\bf Faster rate under further structure.} We derived a tight set-level mismatch bound of order $\mathcal{O}(d/T)$, later extended to an action-level result using the incenter estimator. It remains an open question whether faster rates are achievable under additional assumptions, e.g., strong convexity/smoothness of the feature map or distributional assumptions on the context.

{\bf Active learning.} Our analysis assumes the contexts $\{s_t\}_{t=1}^T$ are drawn i.i.d.\ from a fixed distribution $\mathbb P_\mathcal{S}$. However, in many interactive settings, the learner can adaptively choose which contexts to query next, suggesting an active variant of IO. Designing $s_{t+1}$ based on past observations to shrink the consistency set $\mathcal{C}_T$ more rapidly is a promising direction.

{\bf Noisy dataset.} Although our IO problem has a stochastic relation with the environment, it does not take any additive noise into account. Since additive noise bridges IO and online learning in more realistic scenarios, extending the analysis to noisy observations is an important next step.

{\bf Computational complexity vs. generalization.} We proposed a parameter-free algorithm with lower per-iteration cost than generic solvers for computing $\hat\theta_T^{\mathrm{sub}}$. We empirically observed that this method exhibits comparable generalization behavior, but we have not yet provided a corresponding theory. Establishing such a formal tradeoff between computational complexity and the resulting generalization guarantees is also an important future direction.
\clearpage
\section*{Acknowledgments}
Pouria Fatemi and Suvrit Sra acknowledge generous support from the Alexander von Humboldt Foundation. Hoomaan Maskan acknowledges support from the Science for Life Laboratory (SciLifeLab), and the Wallenberg AI, Autonomous Systems, and Software Program (WASP), funded by the Knut and Alice Wallenberg Foundation. Peyman Mohajerin Esfahani acknowledges the support by the ERC Starting Grant TRUST-949796. The paper is based primarily on the work done while Hoomaan Maskan was at Umeå University. The authors declare no competing interests. 

\bibliographystyle{icml2026}
\bibliography{ref}

\clearpage
\appendix

\section{Background and Helpful Lemmas}
\label[appendix]{app:discussions}
In this section, we present some of the important lemmas and background used in this work.

\subsection{Chance constraints and scenario programs}
\label[appendix]{subsec:ccp-scp}

We now recall finite-sample scenario program and present its corresponding chance-constrained optimization.
Let $\Theta \subset \mathbb{R}^d$ be a convex set and let 
$J:{\Theta} \rightarrow \mathbb{R}$ be a convex function. 
Consider a measurable function 
$g : \Theta \times \mathcal{S} \to \mathbb{R}$ 
that is convex in $\theta$ for every $s \in \mathcal{S}$.
The associated chance-constrained program is
\begin{align}
\tag{CCP$_\varepsilon$}
\left\{
\begin{array}{ll}
\displaystyle \min_{\theta} & J(\theta) \\[2pt]
\text{s.t.} & \mathbb{P}_{\mathcal{S}}\big[g(\theta,s) \le 0\big] \,\ge\, 1-\varepsilon,\\
            & \theta \in \Theta,
\end{array}
\right.
\label{equ:ccp}
\end{align}
where $\varepsilon \in [0,1]$ is the violation level and
$\mathbb{P}_{\mathcal{S}}$ denotes the distribution of states.
Let $(s_i)_{i=1}^T$ be $T$ i.i.d.\ samples from $\mathbb{P}_{\mathcal{S}}$.  
The corresponding \emph{scenario program} is
\begin{align}
\tag{SCP}
\left\{
\begin{array}{ll}
\displaystyle \min_{\theta} & J(\theta) \\[4pt]
\text{s.t.} & g(\theta,s_i) \le 0,\quad i = 1,\dots,T, \\[2pt]
            & \theta \in \Theta.
\end{array}
\right.
\label{equ:scp}
\end{align}
Since \eqref{equ:scp} depends on the random samples $(s_i)_{i=1}^T$, its optimizer is a random variable. We assume that \eqref{equ:scp}  admits a unique solution with probability 1. The following theorem states that if $T$ is large enough, the unique solution of \eqref{equ:scp} is a feasible solution of \eqref{equ:ccp} with high probability \citep{campi2008exact}.
\begin{theorem}[\citep{campi2008exact}]
\label{thm:CG-feasibility}
Let $\beta \in [0,1]$ and $T \ge N(\varepsilon,\beta)$, where
\begin{equation*}
N(\varepsilon,\beta)
:= \min\left\{ n \in \mathbb{N} \;\middle|\;
\sum_{i=0}^{d-1} \binom{n}{i} \varepsilon^{i} (1-\varepsilon)^{n-i}
\le \beta \right\}.
\end{equation*}
Then, with probability at least $1-\beta$ over the random draw of
$(s_i)_{i=1}^T$, the optimizer of \eqref{equ:scp} is a feasible solution of
\eqref{equ:ccp}.
\end{theorem}
\begin{remark}[Low-data regime guarantees]
The condition $T\ge N(\varepsilon,\beta)$ is meaningful only when
$T\ge d$. If $T<d$, then for every $\varepsilon\in(0,1)$,
\[
\sum_{i=0}^{d-1} \binom{T}{i}\,\varepsilon^{i}\,(1-\varepsilon)^{T-i}
=
\sum_{i=0}^{T} \binom{T}{i}\,\varepsilon^{i}\,(1-\varepsilon)^{T-i}
=1,
\]
so the defining inequality for $N(\varepsilon,\beta)$ cannot hold for any $\beta<1$. Equivalently,
when $T<d$ the theorem can only be stated with $\varepsilon=1$, making the guarantee
non-informative.
\end{remark}
The following lemma reformulates the result from \Cref{thm:CG-feasibility} for any $T=\Omega(d +  \log(1/\beta))$.
\begin{lemma}
\label{lem:epsilon_T_order}
Fix $\beta\in(0,1)$. For any $T\ge 2\big(d + \log(1/\beta)\big)$, the choice
\(
\varepsilon  = \dfrac{2}{T}\big(d + \log(1/\beta)\big)
\)
satisfies $T\ge N(\varepsilon,\beta)$.
\end{lemma}
The proof of this lemma is given in \Cref{app:lem_epsilon_T_order}. Next, we present the definition of support constraint.
\begin{definition}[Support constraint]
    A constraint $g(\theta,s_i) \le 0$, for $i=1,\ldots , T$, is called a support constraint of \eqref{equ:scp} if its removal changes the solution of \eqref{equ:scp}.
\end{definition}

We can now present the notion of a fully supported problem.

\begin{definition}[Fully supported problem]\label{def:fully-supp}
    The scenario program \eqref{equ:scp} with $T\geq d$ is called fully supported if the number of its support constraints is exactly $d$. The chance constraint program \eqref{equ:ccp} is fully supported if for $T\geq d$, its corresponding \eqref{equ:scp} is fully supported with probability 1.
\end{definition}

\subsection{Online protocol}
\label[appendix]{app:online-protocol}

We make the dependence on the history explicit for the proofs.
Let $\mathcal{F}_t:=\sigma\!\big((s_1,a_1^\star),\dots,(s_t,a_t^\star)\big)$ be the sigma-algebra generated by the
expert demonstrations up to round $t$ (and set $\mathcal{F}_0$ to be the trivial $\sigma$-algebra).
At round $t\ge 1$, the learner computes $\hat\theta^{\mathrm{in}}_{t-1}$ from $D_{t-1}:=\{(s_i,a_i^\star)\}_{i=1}^{t-1}$,
plays $a_t:=a_{\hat\theta^{\mathrm{in}}_{t-1}}(s_t)$, and observes $a_t^\star=a_{\theta^\star}(s_t)$.
We have $s_t\sim \mathbb{P}_{\mathcal{S}}$ and $s_t\perp\!\!\!\perp \mathcal{F}_{t-1}$.

Under Assumption~\ref{ass:bounded-features}, define $0\le r_t\le \|\theta^\star\|_2\,B$ for all $t$.

For $N\ge 1$, let $\hat\theta^{\mathrm{in}}_N$ be the incenter estimator from $N$ i.i.d.\ demonstrations, and define
\[
X_N
:=
\mathbb{P}_{\mathcal{S}}\!\big(a_{\hat\theta^{\mathrm{in}}_N}(s)\neq a_{\theta^\star}(s)\big)\in[0,1].
\]
In particular, in the online protocol we write
$X_{t-1}:=\mathbb{P}_{\mathcal{S}}(a_{\hat\theta^{\mathrm{in}}_{t-1}}(s)\neq a_{\theta^\star}(s))$. To bound the expected cumulative regret, we need to bound the expectation of $X_N$ using our generalization result. The proof of this result given below can be found in \Cref{app:lem_EN-incenter}. 

\begin{lemma}[Expected mismatch probability for the incenter estimator]
\label{lem:EN-incenter}
Assume
$\mathrm{int}(\mathcal{C}_N)\neq\emptyset$ almost surely. Let $X_N$ be as above.
Then for all $N\ge d$,
\begin{equation}
\label{eq:EN-bound-incenter}
\mathbb{E}_D[X_N]\ \le\ \frac{d}{N + 1},
\end{equation}
where $\mathbb{E}_D$ denotes expectation over the $N$ i.i.d.\ demonstrations used to construct
$\hat\theta^{\mathrm{in}}_N$.
\end{lemma}

To achieve a high probability bound on cumulative regret, we use Freedman's bound from the martingale concentration inequality literature.
\begin{lemma}[Freedman's inequality {\citep{freedman1975tail,dzhaparidze2001bernstein}}]
\label{lem:freedman}
Let $(M_t)_{t\ge 0}$ be a martingale with $M_0=0$ and differences
$Z_t:=M_t-M_{t-1}$ satisfying $|Z_t|\le b$ a.s.
Let the predictable quadratic variation be
\[
V_t:=\sum_{i=1}^t \mathbb{E}[Z_i^2\mid \mathcal{F}_{i-1}].
\]
Then, for any fixed $t\ge 1$, any deterministic $v\ge 0$, and any
$\delta\in(0,1)$,
\[
\mathbb{P}\!\left(
M_t
>
\sqrt{2v\log\!\frac{1}{\delta}}
+\frac{2}{3}b\log\!\frac{1}{\delta}
\ \text{ and }\ 
V_t\le v
\right)
\le \delta.
\]
\end{lemma}

\subsection{A classical VC-dimension view of IO}\label[appendix]{app:VC_IO}

Noiseless inverse optimization admits an exact realizable binary classification interpretation for the set-level mismatch. The dataset
\(
D_T=\{(s_t,a_t^\star)\}_{t=1}^T
\)
may be viewed as an i.i.d.\ sample of supervised examples drawn from the distribution of the random pair
\(
(s,a_{\theta^\star}(s))
\)
with \(s\sim \mathbb P_{\mathcal S}\). For each \(\theta\in\mathbb R^d\), define the induced binary classifier
\[
h_\theta(s,a):=\mathbf 1\{a\in \mathcal A_\theta(s)\}
=\mathbf 1\{\ell_\theta^{\mathrm{sub}}(s,a)=0\},
\]
and let
\[
\mathcal H:=\{h_\theta:\theta\in\mathbb R^d\}.
\]
Thus, \(h_\theta(s,a)=1\) if and only if action \(a\) is greedy under parameter \(\theta\) in state \(s\). Under this identification, the population risk is
\[
R(\theta)
:= \mathbb P_{\mathcal S}\bigl(h_\theta(s,a_{\theta^\star}(s)) = 0\bigr)
= \mathbb P_{\mathcal S}\bigl(a_{\theta^\star}(s)\notin \mathcal A_\theta(s)\bigr)
= \mathbb P_{\mathcal S}\!\left(\ell_\theta^{\mathrm{sub}}(s,a_{\theta^\star}(s))>0\right),
\]
which is exactly the set-level mismatch probability studied in this paper. Since
\(
R(\theta^\star)=0,
\)
the induced classification problem is realizable.

Likewise, the empirical classification error is
\[
\widehat R_T(\theta)
:=
\frac{1}{T}\sum_{t=1}^T \mathbf 1\{a_t^\star\notin \mathcal A_\theta(s_t)\},
\]
and
\[
\widehat R_T(\theta)=0
\quad\Longleftrightarrow\quad
\ell_\theta^{\mathrm{sub}}(s_t,a_t^\star)=0
\ \ \forall t\in[T]
\quad\Longleftrightarrow\quad
\theta\in C_T.
\]
Hence, any learning rule \(D_T\mapsto \widehat\theta_T(D_T)\) satisfying
\(
\widehat\theta_T(D_T)\in C_T
\)
is sample-consistent for \(\mathcal H\), in the sense that
\(
\widehat R_T(\widehat\theta_T)=0.
\)
Consequently, whenever
\(
d_{\mathrm{VC}}:=\operatorname{VCdim}(\mathcal H)<\infty,
\)
standard realizable VC theory implies that, for any confidence level \(\delta\in(0,1)\), any such sample-consistent rule satisfies,
\[
R(\widehat\theta_T)
=
\mathcal O\!\left(
\frac{d_{\mathrm{VC}}\log(T/d_{\mathrm{VC}})+\log(1/\delta)}{T}
\right),
\]
with probability at least \(1-\delta\), for \(T\ge d_{\mathrm{VC}}\)~\citep{shalev2014understanding}.

This provides a clean conceptual PAC-learning interpretation of noiseless inverse optimization. At the same time, it also highlights the limitation of a purely VC-based viewpoint. Indeed, the observed supervision is one-sided: each demonstration only certifies that the expert action must remain in the argmax set \(\mathcal A_\theta(s)\) under a candidate parameter \(\theta\), and therefore each sample acts by eliminating inconsistent parameters from the parameter space, thereby forming the consistency set \(C_T\). As a result, the relevant complexity is not naturally expressed in terms of arbitrary labelings of state-action pairs. In particular, the quantity \(d_{\mathrm{VC}}\) is determined by the geometry of the induced hypothesis class \(\mathcal H\), rather than directly by the parameter dimension \(d\). Therefore, while the VC reduction is conceptually useful, it does not by itself explain the dimension-dependent \(O(d/T)\) guarantee established in this paper. This observation motivates the scenario-based view, where each demonstration is treated directly as a random feasibility constraint in parameter space.

\subsection{Numerical experiment setup}\label[appendix]{app:numerical}
We consider a contextual linear decision model with $K=|\mathbb A|=15$ discrete actions and
$d$-dimensional features. The expert follows a greedy policy for a fixed but unknown parameter $\theta^\star\in\R^d$ satisfying
$\sum_{i=1}^d\theta_i^\star=1$. For each $k\in[K]$, we generate a fixed action vector
$a_k\sim\mathcal N(0,I_d)$ once at the beginning of the experiment and normalize it, and set
$\mathbb A=\{a_1,\dots,a_K\}$. At each round, we draw a fresh random linear map
\(
s_t\in\R^{d\times d},\ (s_t)_{ij}\sim \mathrm{Unif}[-3,3],
\)
and define the context features as
\(
\psi(s_t,a_k)=s_t a_k\in\R^d.
\)
The learner observes $\{(s_t,a_t^\star)\}_{t=1}^T$ for $T\le T_{\max}=300$ and uses a greedy plug-in
policy based on an estimate $\hat\theta_T$ to select $a_{\hat\theta_T}(s)$. Ties are broken by choosing the smallest-index action. Throughout,
we set $J(\theta)=\|\theta\|_2^2$ and $\Theta = \{\theta \in \R^d \mid \sum_{i=1}^d\theta_i=1\}$.
We compare three estimators $\hat\theta_T^{\mathrm{sub}}$, $\hat\theta_T^{\mathrm{in}}$, and
$\hat\theta_T^{\mathrm{Pol}}$ derived from \eqref{eq:reg-io-scp}, \eqref{eq:scp-incenter}, and the
parameter-free approach in \Cref{sec:param_free}, respectively.
To directly test generalization, we evaluate the current estimate $\hat\theta_T$ on
$N_{\mathrm{test}}=1000$ fresh contexts and record the empirical mismatch rate
\[
\mathrm{GE}_T
:= \frac{1}{N_{\mathrm{test}}}\sum_{j=1}^{N_{\mathrm{test}}}
\mathbb{I}\!\left\{a_{\hat\theta_T}(s_j)\neq a_{\theta^\star}(s_j)\right\}.
\]
We repeat the full $T_{\max}$-round experiment for $10$ independent Monte-Carlo runs. For each curve,
we report the mean across runs together with a two-sided $90\%$ normal confidence band. The convex
programs \eqref{eq:reg-io-scp} and \eqref{eq:scp-incenter} are solved using CVXPY~\citep{diamond2016cvxpy}. All experiments were run on a personal computer and required no specialized hardware.

\clearpage

\section{Proofs}\label[appendix]{app:proofs}
\subsection{Proof of \Cref{lem:epsilon_T_order}}\label[appendix]{app:lem_epsilon_T_order}
Recall that
\[
\sum_{i=0}^{d-1}\binom{T}{i}\varepsilon^{i}(1-\varepsilon)^{T-i}
=\mathbb P[X\le d-1],
\qquad X\sim\mathrm{Bin}(T,\varepsilon),
\]
with mean $\mu:=\mathbb E[X]=T\varepsilon$. For our choice of $\varepsilon$ we have
$\mu=T\varepsilon=2\big(d+\log(1/\beta)\big)>d-1$, hence
$\delta:=1-\frac{d-1}{\mu}\in(0,1)$ and $(1-\delta)\mu=d-1$. By the Chernoff lower-tail bound,
\[
\mathbb P[X\le d-1]
=\mathbb P\big[X\le (1-\delta)\mu\big]
\le \exp\!\left(-\frac{\delta^2\mu}{2}\right)
=\exp\!\left(-\frac{(\mu-d+1)^2}{2\mu}\right).
\]
Now, if $T\varepsilon\ge 2\big(\log(1/\beta)+d\big)$, then in particular
$T\varepsilon\ge 2\big(\log(1/\beta)+d-1\big)$ and therefore
\[
\frac{(T\varepsilon-d+1)^2}{2T\varepsilon}
=
\frac{(T\varepsilon-(d-1))^2}{2T\varepsilon}
\ge
\frac{T\varepsilon}{2}-(d-1)
\ge \log(1/\beta).
\]
Thus $\mathbb P[X\le d-1]\le \exp(-\log(1/\beta))=\beta$, i.e.,
\[
\sum_{i=0}^{d-1}\binom{T}{i}\varepsilon^{i}(1-\varepsilon)^{T-i}\le \beta,
\]
which means $T\ge N(\varepsilon,\beta)$ by definition of $N(\varepsilon,\beta)$.

\subsection{Proof of \Cref{lem:EN-incenter}}\label[appendix]{app:lem_EN-incenter}

To compute $\mathbb E_D[X_N]$, use the tail-integral formula:
\[
\mathbb E_D[X_N]
=
\int_0^1 \mathbb P_D(X_N>\varepsilon)\,d\varepsilon.
\]
By \Cref{prop:margin-mismatch},
\[
\mathbb P_D(X_N>\varepsilon)
\le
\sum_{i=0}^{d-1}\binom{N}{i}\varepsilon^i(1-\varepsilon)^{N-i}.
\]
Hence
\[
\mathbb E_D[X_N]
\le
\sum_{i=0}^{d-1}\binom{N}{i}
\int_0^1 \varepsilon^i(1-\varepsilon)^{N-i}\,d\varepsilon.
\]
The integral is a Beta integral:
\[
\int_0^1 \varepsilon^i(1-\varepsilon)^{N-i}\,d\varepsilon
=
\frac{i!(N-i)!}{(N+1)!}.
\]
Therefore
\[
\binom{N}{i}\cdot\frac{i!(N-i)!}{(N+1)!}
=
\frac{N!}{i!(N-i)!}\cdot \frac{i!(N-i)!}{(N+1)!}
=
\frac{1}{N+1}.
\]
Summing over $i=0,\dots,d-1$, we obtain
\[
\mathbb E_D[X_N] \le \sum_{i=0}^{d-1} \frac{1}{N+1} = \frac{d}{N+1},
\]
which implies \eqref{eq:EN-bound-incenter}.

\subsection{Proof of \Cref{thm:scenario-nonexclusion-unique}}\label[appendix]{app:prf_scenario-nonexclusion-unique}

Consider the chance-constrained program
\begin{equation}
\label{eq:reg-io-ccp}
\min_{\theta\in\Theta}\ J(\theta)
\quad\text{s.t.}\quad
\mathbb{P}_{\mathcal{S}}\big(\ell_\theta^{\mathrm{sub}}(s,a_{\theta^\star}(s))\le 0\big)\ \ge\ 1-\varepsilon.
\end{equation}
For a small $\varepsilon$, any feasible solution of \eqref{eq:reg-io-ccp} is a generalizable result. 
Since $a_t^\star=a_{\theta^\star}(s_t)$, the constraints in~\eqref{eq:reg-io-scp} are exactly the
scenario constraints obtained by sampling $s_t\sim\mathbb{P}_{\mathcal{S}}$ and enforcing
$\ell_\theta^{\mathrm{sub}}(s_t,a_{\theta^\star}(s_t))\le 0$ for all $t\in[T]$.
Therefore, we use the result from \citep[Theorem 1]{campi2008exact}. This theorem states that if $T\ge N(\varepsilon,\beta)$, then with probability at least $1-\beta$ over the draw of
$(s_t)_{t=1}^T$, the unique optimizer $\hat\theta_T^{\mathrm{sub}}$ of the scenario program \eqref{eq:reg-io-scp} is feasible for the chance constraint problem 
\eqref{eq:reg-io-ccp}, i.e.,

\begin{equation}
\label{eq:chance-feasible-unique}
\mathbb{P}_{\mathcal{S}}\big(\ell_{\hat\theta_T^{\mathrm{sub}}}^{\mathrm{sub}}(s,a_{\theta^\star}(s))\le 0\big)\ \ge\ 1-\varepsilon.
\end{equation}

By definition, $\ell_\theta^{\mathrm{sub}}(s,a_{\theta^\star}(s))\ge 0$ for all $(\theta,s)$.
Therefore, the event $\{\ell_{\theta}^{\mathrm{sub}}(s,a_{\theta^\star}(s))\le 0\}$ is equivalent to $\{\ell_{\theta}^{\mathrm{sub}}(s,a_{\theta^\star}(s))=0\}$. Moreover,
by~\eqref{eq:suboptimality-gap}, \(\ell_\theta^{\mathrm{sub}}(s,a_{\theta^\star}(s))=0\) is equivalent to \(a_{\theta^\star}(s)\in \mathcal{A}_{\theta}(s).\)
Applying this equivalence to~\eqref{eq:chance-feasible-unique} yields:
\[
\mathbb{P}_{\mathcal{S}}\big(a_{\theta^\star}(s)\in \mathcal{A}_{\hat\theta_T^{\mathrm{sub}}}(s)\big)\ \ge\ 1-\varepsilon,
\]
which is equivalent to~\eqref{eq:nonexclusion-unique}.
It is also possible to make the choice of $\varepsilon$ explicit as a function of $T$. In particular, for any $T\ge 2\big(d+\log(1/\beta)\big)$, the choice
\(
\varepsilon = \frac{2}{T}\big(d + \log(1/\beta)\big)
\)
satisfies $T\ge N(\varepsilon,\beta)$. See \Cref{lem:epsilon_T_order} for the proof.

\subsection{Proof of \Cref{prop:margin-mismatch}}\label[appendix]{app:prf_margin-mismatch}

Consider the chance-constrained problem
\[
\min_{\theta\in\mathbb R^d} J(\theta)
\quad\text{s.t.}\quad
\mathbb{P}_{\mathcal{S}}\!\left(\ell_\theta^{\mathrm{in}}(s,a_{\theta^\star}(s)) \le 0\right)\ge 1-\varepsilon,
\]
whose associated scenario program is \eqref{eq:scp-incenter}. For brevity, write
$\hat\theta:=\hat\theta_T^{\mathrm{in}}$ and
\(
\ell(s):=\ell_{\hat\theta}^{\mathrm{in}}\big(s,a_{\theta^\star}(s)\big).
\)
By \citep[Theorem~1]{campi2008exact}, if $T\ge N(\varepsilon,\beta)$, with probability
at least $1-\beta$ over the draw of the training sample $D_T$,
\begin{equation}
\label{eq:chance-feasible}
\mathbb{P}_{\mathcal{S}}\!\left(\ell(s)\le 0\right)\ \ge\ 1-\varepsilon.
\end{equation}
We now show that, for every $s\in\mathcal S$,
\begin{equation}
\label{eq:goodset_implies_joint}
\ell(s)\le 0
\ \Longrightarrow\
\Big(\mathcal A_{\hat\theta_T^{\mathrm{in}}}(s)
=
\{a_{\theta^\star}(s)\}\Big).
\end{equation}
Fix any $s$ with $\ell(s)\le 0$. By the definition of $\ell$ in \eqref{eqn:incenter_loss},
for every $a\in\actionset{s}$,
\begin{equation}
\label{eq:pointwise-margin}
\langle \hat\theta, \delta(s,a)\rangle \le -\|\delta(s,a)\|.
\end{equation}
For $a=a_{\theta^\star}(s)$ we have $\delta(s,a)=0$, hence \eqref{eq:pointwise-margin} holds
with equality. 

If $a\neq a_{\theta^\star}(s)$, uniqueness of the expert's optimum
implies $\langle\theta^\star,\delta(s,a)\rangle<0$, and hence
$\delta(s,a)\neq0$, so $\|\delta(s,a)\|>0$ and \eqref{eq:pointwise-margin} implies:
$\langle \hat\theta,\delta(s,a)\rangle<0$. Equivalently,
\[
\langle \hat\theta,\psi(s,a)\rangle
<
\langle \hat\theta,\psi(s,a_{\theta^\star}(s))\rangle,
\qquad \forall a\neq a_{\theta^\star}(s).
\]
Thus $a_{\theta^\star}(s)$ is the unique maximizer of $a\mapsto\langle \hat\theta,\psi(s,a)\rangle$,
which proves \eqref{eq:goodset_implies_joint}.
Combining \eqref{eq:chance-feasible} and \eqref{eq:goodset_implies_joint} yields, on the same
event of probability at least $1-\beta$ over $D_T$,
\begin{align*}
\mathbb{P}_{\mathcal {S}}\!\left(\mathcal A_{\hat\theta_T^{\mathrm{in}}}(s)
=
\{a_{\theta^\star}(s)\}
\right) \ge
\mathbb{P}_{\mathcal{S}}\!\left(\ell(s)\le 0\right) \ge 1-\varepsilon.
\end{align*}
This is the claimed guarantee.

\subsection{Proof of \Cref{thm:instantaneous-regret-incenter}}
\label[appendix]{app:online-proof-main}

\begin{proposition}[Regret upper bound (complete version)]
Suppose \Cref{ass:bounded-features} holds and
$\operatorname{int}(\mathcal C_t)\neq\emptyset$ almost surely
for every $t\in\{1,\ldots,T-1\}$. Consider the
stochastic online protocol based on the incenter estimator, where
$a_t=a_{\hat\theta^{\mathrm{in}}_{t-1}}(s_t)$. Fix $\beta\in(0,1)$ and let
$\varepsilon\in(0,1]$ satisfy $t-1\ge N(\varepsilon,\beta)$. Then, with
probability at least $1-\beta$ over the draw of $D_{t-1}$,
\begin{align}
\label{eq:instantaneous-regret-hp}
    \mathbb{E}[r_t\mid D_{t-1}]
    \le
    \|\theta^\star\|_2\,B\,\varepsilon .
\end{align}
In particular, for every $t\ge d+1$,
\begin{align}
\label{eq:instantaneous-regret-expectation}
    \mathbb{E}[r_t]
    \le
    \|\theta^\star\|_2\,B\,\frac{d}{t}.
\end{align}
Consequently, for every $T\ge d+1$,
\begin{align}
\label{eq:online-cum-regret-expectation}
    \mathbb{E}[R_T]
    =
    \mathcal{O}\!\left(
    \|\theta^\star\|_2\,B\,d
    \left(1+\log \frac{T}{d}\right)
    \right).
\end{align}
Moreover, for every $\delta\in(0,1)$ and $T\ge d+1$, with probability at least
$1-\delta$ over the online trajectory,
\begin{equation}
\label{eq:online-hp-regret-O}
R_T
=
\mathcal{O}\!\left(
\|\theta^\star\|_2\,B\,
\left(d+\log\frac{T}{\delta}\right)
\left(1+\log \frac{T}{d}\right)
\right).
\end{equation}
\end{proposition}

\begin{proof}
By \Cref{ass:bounded-features}, we have
$0\le r_t\le \|\theta^\star\|_2\,B$ for all $t$.
Moreover, $r_t=0$ whenever $a_t=a_t^\star$.
Conditioning on $D_{t-1}$, $\hat\theta^{\mathrm{in}}_{t-1}$ is fixed, and
$s_t\sim\mathbb{P}_{\mathcal S}$ independently of $D_{t-1}$. 
Then,
\begin{align}
\label{eq:cond-rt-incenter}
\begin{aligned}
\mathbb{E}[r_t\mid D_{t-1}]
&=
\mathbb{E}\!\left[
r_t\,\mathbf{1}\{a_t\neq a_t^\star\}\,\middle|\,D_{t-1}
\right] \\
&\le
\|\theta^\star\|_2\,B\,
\mathbb{P}_{\mathcal S}\!\left(
a_{\hat\theta^{\mathrm{in}}_{t-1}}(s)\neq a_{\theta^\star}(s)
\right)
=
\|\theta^\star\|_2\,B\,X_{t-1}.
\end{aligned}
\end{align}
By \Cref{prop:margin-mismatch}, if $t-1\ge N(\varepsilon,\beta)$, then
$X_{t-1}\le \varepsilon$ with probability at least $1-\beta$ over $D_{t-1}$.
This proves \eqref{eq:instantaneous-regret-hp}.

Taking expectations in \eqref{eq:cond-rt-incenter} yields
\[
\mathbb{E}[r_t]
\le
\|\theta^\star\|_2\,B\,\mathbb{E}_D[X_{t-1}].
\]
For $t\ge d+1$, \Cref{lem:EN-incenter} gives
\[
\mathbb{E}_D[X_{t-1}]
\le
\frac{d}{t},
\]
and hence
\[
\mathbb{E}[r_t]
\le
\|\theta^\star\|_2\,B\,\frac{d}{t}.
\]
This proves \eqref{eq:instantaneous-regret-expectation}.

For cumulative regret, split the sum at $d$:
\[
\mathbb{E}[R_T]
=
\sum_{t=1}^{\min\{d,T\}}\mathbb{E}[r_t]
+
\sum_{t=d+1}^{T}\mathbb{E}[r_t].
\]
For $t\le d$, $\mathbb{E}[r_t]\le \|\theta^\star\|_2\,B$, and therefore
\[
\sum_{t=1}^{\min\{d,T\}}\mathbb{E}[r_t]
\le
d\,\|\theta^\star\|_2\,B.
\]
For $t\ge d+1$, using \eqref{eq:instantaneous-regret-expectation},
\[
\sum_{t=d+1}^{T}\mathbb{E}[r_t]
\le
\|\theta^\star\|_2\,B\,
d\sum_{t=d+1}^{T}\frac{1}{t}
\le
\|\theta^\star\|_2\,B\,d\log\frac{T}{d}.
\]
Thus,
\[
\mathbb{E}[R_T]
=
\mathcal{O}\!\left(
\|\theta^\star\|_2\,B\,d
\left(1+\log\frac{T}{d}\right)
\right),
\]
which proves \eqref{eq:online-cum-regret-expectation}.

It remains to prove the high-probability cumulative regret bound. Define
\(
\mu_t:=\mathbb{E}[r_t\mid \mathcal{F}_{t-1}]
\)
and the martingale
\[
M_t:=\sum_{i=1}^t (r_i-\mu_i),\qquad M_0=0.
\]
Then
\begin{equation}
\label{eq:RT-decomp-incenter}
R_T=\sum_{t=1}^T \mu_t + M_T.
\end{equation}
From \eqref{eq:cond-rt-incenter}, we have
\(
\mu_t\le \|\theta^\star\|_2\,B\,X_{t-1}.
\)

Fix $\delta\in(0,1)$ and set
\[
\bar\beta:=\frac{\delta}{2T},
\qquad
t_0
:=
\min\!\Big\{
T+1,\ 
\big\lceil 2\big(d+\log\tfrac{2T}{\delta}\big)\big\rceil+1
\Big\}.
\]
For any $t\in\{t_0,\dots,T\}$, if this index set is nonempty, we have
\(
t-1\ge 2\big(d+\log(1/\bar\beta)\big).
\)
Applying \Cref{prop:margin-mismatch} with sample size $t-1$ and confidence
$\bar\beta$, together with \Cref{lem:epsilon_T_order}, gives
\[
\mathbb{P}\!\Big(
X_{t-1}
\le
2\,\frac{d+\log(1/\bar\beta)}{t-1}
\Big)
\ge
1-\bar\beta,
\qquad t=t_0,\dots,T.
\]
A union bound gives an event $E$ with $\mathbb{P}(E)\ge 1-\delta/2$ on which
\begin{equation}
\label{eq:uniform-X-incenter}
\forall\,t\in\{t_0,\dots,T\}:\ 
X_{t-1}
\le
2\,\frac{d+\log\frac{2T}{\delta}}{t-1}.
\end{equation}
For $t<t_0$ we use the trivial bound $X_{t-1}\le 1$. Therefore, on $E$,
\begin{equation}
\label{eq:sum-mu-incenter}
\sum_{t=1}^T \mu_t
\le
\|\theta^\star\|_2\,B
\Bigg[
(t_0-1)
+
2\Big(d+\log\frac{2T}{\delta}\Big)
\sum_{t=t_0}^T \frac{1}{t-1}
\Bigg],
\end{equation}
where the sum is interpreted as zero if $t_0=T+1$.

Let $Z_t:=M_t-M_{t-1}=r_t-\mu_t$. Then
$|Z_t|\le \|\theta^\star\|_2\,B$ almost surely. Moreover,
\[
\mathbb{E}[Z_t^2\mid \mathcal{F}_{t-1}]
=
\operatorname{Var}(r_t\mid \mathcal{F}_{t-1})
\le
\mathbb{E}[r_t^2\mid \mathcal{F}_{t-1}]
\le
(\|\theta^\star\|_2\,B)^2 X_{t-1}.
\]
Thus, on $E$, the predictable quadratic variation satisfies
\begin{equation}
\label{eq:VT-incenter}
V_T
:=
\sum_{t=1}^T
\mathbb{E}[Z_t^2\mid \mathcal{F}_{t-1}]
\le
(\|\theta^\star\|_2\,B)^2
\Bigg[
(t_0-1)
+
2\Big(d+\log\frac{2T}{\delta}\Big)
\sum_{t=t_0}^T \frac{1}{t-1}
\Bigg].
\end{equation}

Set
\[
H_T
:=
(t_0-1)
+
2\Big(d+\log\frac{2T}{\delta}\Big)
\sum_{t=t_0}^T \frac{1}{t-1}.
\]
Applying \Cref{lem:freedman} with confidence $\delta/2$,
$b=\|\theta^\star\|_2B$, and 
\(
v=(\|\theta^\star\|_2B)^2 H_T
\),
gives
\[
\mathbb{P}\!\left(
M_T
>
\|\theta^\star\|_2 B
\sqrt{2H_T\log\frac{2}{\delta}}
+
\frac{2}{3}\|\theta^\star\|_2 B\log\frac{2}{\delta}
\ \text{ and }\ 
V_T\le (\|\theta^\star\|_2B)^2H_T
\right)
\le \frac{\delta}{2}.
\]
Since, by \eqref{eq:VT-incenter}, the event $E$ implies
$V_T\le(\|\theta^\star\|_2B)^2H_T$, we have
\[
\mathbb{P}\!\left(
E\cap
\left\{
M_T
>
\|\theta^\star\|_2 B
\sqrt{2H_T\log\frac{2}{\delta}}
+
\frac{2}{3}\|\theta^\star\|_2 B\log\frac{2}{\delta}
\right\}
\right)
\le \frac{\delta}{2}.
\]
Since also $\mathbb{P}(E^c)\le \delta/2$, a union bound gives that, with probability at least
$1-\delta$, the event $E$ holds and
\[
M_T
\le
\|\theta^\star\|_2 B
\sqrt{2H_T\log\frac{2}{\delta}}
+
\frac{2}{3}\|\theta^\star\|_2 B\log\frac{2}{\delta}.
\]
On this event, combining
\eqref{eq:RT-decomp-incenter}, \eqref{eq:sum-mu-incenter}, and
\eqref{eq:VT-incenter} gives
\begin{align}
\label{eq:hp-regret}
R_T
&\le
\|\theta^\star\|_2 B
\Bigg[
(t_0-1)
+
2\Big(d+\log\!\frac{2T}{\delta}\Big)
\sum_{t=t_0}^T \frac{1}{t-1}
\Bigg]
\nonumber\\
&\quad+
\|\theta^\star\|_2 B
\sqrt{
2
\Bigg[
(t_0-1)
+
2\Big(d+\log\!\frac{2T}{\delta}\Big)
\sum_{t=t_0}^T \frac{1}{t-1}
\Bigg]
\log\!\frac{2}{\delta}
}
\nonumber\\
&\quad+
\frac{2}{3}\,\|\theta^\star\|_2 B\,
\log\!\frac{2}{\delta}.
\end{align}
Finally, since
\[
t_0-1
=
\mathcal{O}\!\left(d+\log\frac{T}{\delta}\right)
\qquad\text{and}\qquad
\sum_{t=t_0}^T \frac{1}{t-1}
\le
1+\log\frac{T}{d},
\]
we obtain
\[
R_T
=
\mathcal{O}\!\left(
\|\theta^\star\|_2\,B\,
\left(d+\log\frac{T}{\delta}\right)
\left(1+\log\frac{T}{d}\right)
\right),
\]
which proves \eqref{eq:online-hp-regret-O}.
\end{proof}

\subsection{Proof of \Cref{prop:minimax}}
\label[appendix]{proof:minimax}
The proof is a direct reduction to realizable binary classification. Consider
$$
A(s)=\{0,1\},
\qquad
\psi(s,0)=0,
\qquad
\psi(s,1)=s.
$$
Under a fixed tie-breaking rule,
$$
a_\theta(s) =
\mathbf 1\{\langle\theta,s\rangle\ge0\}.
$$
The resulting expert-policy class is the class of homogeneous linear classifiers in $\mathbb R^d$, whose VC dimension is $d$. The standard realizable PAC lower bound applies to arbitrary-output, possibly randomized learners and gives sample complexity
$$
\Omega\left(
\frac{d+\log(1/\delta)}{\varepsilon}
\right)
$$
see \citet{ehrenfeucht1989general}.

\subsection{Proof of \Cref{thm:scenario-tightness-regret}}
\label[appendix]{proof:scenario-tightness-regret}

Consider the IO instance constructed in
\Cref{thm:scenario-tightness-continuous}. In this construction,
$\theta^\star=(1,\bm 0)$ and $a_{\theta^\star}(s)=0$ for every
$s\in\mathcal S$. Moreover,
\[
F_{\theta^\star}(s,a)=-2|a|+a,
\]
so
\[
F_{\theta^\star}(s,0)-F_{\theta^\star}(s,1)=1,
\qquad
F_{\theta^\star}(s,0)-F_{\theta^\star}(s,-1)=3.
\]

Fix a round $t\ge d+1$ and condition on $D_{t-1}$. Then
$\hat\theta_{t-1}^{\mathrm{sub}}$ is fixed. Write
$\hat\theta:=\hat\theta_{t-1}^{\mathrm{sub}}=(1,\hat\theta_{-1})$ and 
for a fresh state $s$, let $x=s^\top\hat\theta_{-1}$. From the construction in
\Cref{thm:scenario-tightness-continuous}, the learned score is
\[
F_{\hat\theta}(s,a)=-2|a|+a+a x .
\]
Hence the unique greedy action under $\hat\theta$ is $1$ when $x>1$, is $-1$
when $x<-3$, and is $0$ when $-3<x<1$. The boundary events $\{x=1\}$ and
$\{x=-3\}$ have $\mathbb{P}_{\mathcal S}$-measure zero, since
$\mathbb{P}_{\mathcal S}$ is absolutely continuous on the sphere. Therefore,
$\mathbb{P}_{\mathcal S}$-almost surely, the action-level mismatch event
$\{a_{\hat\theta}(s)\neq a_{\theta^\star}(s)\}$ implies
$a_{\hat\theta}(s)\in\{1,-1\}$. On this event, the regret is either $1$ or $3$,
and hence is at least $1$. Consequently,
\[
r_t(s)
\ge
\mathbf{1}\!\left\{
a_{\hat\theta_{t-1}^{\mathrm{sub}}}(s)\neq a_{\theta^\star}(s)
\right\}
\qquad
\mathbb{P}_{\mathcal S}\text{-a.s.}
\]
Since $s_t\sim\mathbb{P}_{\mathcal S}$ independently of $D_{t-1}$, it follows
that
\begin{align}
\label{eq:regret-lower-cond}
\mathbb{E}[r_t\mid D_{t-1}]
&\ge
\mathbb{P}_{\mathcal S}\!\left(
a_{\hat\theta_{t-1}^{\mathrm{sub}}}(s)\neq a_{\theta^\star}(s)
\right)
=
X_{t-1}.
\end{align}

Now apply \Cref{thm:scenario-tightness-action-level} with sample size $t-1$.
For the constructed IO instance, and for every $J\in\mathcal J$,
\[
\mathbb{P}_{D_{t-1}}\!\left(
X_{t-1}>\varepsilon
\right)
\ge
\sum_{i=0}^{d-1}\binom{t-1}{i}
\varepsilon^{i}(1-\varepsilon)^{t-1-i}.
\]
Combining this with \eqref{eq:regret-lower-cond} gives
\[
\mathbb{P}_{D_{t-1}}\!\left(
\mathbb{E}[r_t\mid D_{t-1}]\ge \varepsilon
\right)
\ge
\sum_{i=0}^{d-1}\binom{t-1}{i}
\varepsilon^{i}(1-\varepsilon)^{t-1-i}.
\]
Taking the infimum over $J\in\mathcal J$ and then the supremum over
$(\theta^\star,\mathbb{P}_{\mathcal S})$ proves the first claim.

We now lower bound the expected instantaneous regret. By
\eqref{eq:regret-lower-cond},
\[
\mathbb{E}[r_t]\ge \mathbb{E}_{D_{t-1}}[X_{t-1}].
\]
Using the tail-integral formula and the action-level lower bound above,
\[
\mathbb{E}_{D_{t-1}}[X_{t-1}]
=
\int_0^1
\mathbb{P}_{D_{t-1}}(X_{t-1}>\varepsilon)\,d\varepsilon
\ge
\sum_{i=0}^{d-1}\binom{t-1}{i}
\int_0^1
\varepsilon^i(1-\varepsilon)^{t-1-i}\,d\varepsilon.
\]
The integral is a Beta integral:
\[
\int_0^1
\varepsilon^i(1-\varepsilon)^{t-1-i}\,d\varepsilon
=
\frac{i!(t-1-i)!}{t!}.
\]
Therefore,
\[
\binom{t-1}{i}
\frac{i!(t-1-i)!}{t!}
=
\frac{1}{t}.
\]
Summing over $i=0,\dots,d-1$, we obtain
\[
\mathbb{E}_{D_{t-1}}[X_{t-1}]
\ge
\frac{d}{t}.
\]
Hence, for every $t\ge d+1$,
\[
\mathbb{E}[r_t]\ge \frac{d}{t}.
\]

Finally, since regret is nonnegative,
\[
\mathbb{E}[R_T]
=
\sum_{t=1}^T\mathbb{E}[r_t]
\ge
\sum_{t=d+1}^T\mathbb{E}[r_t]
\ge
d\sum_{t=d+1}^{T}\frac{1}{t}.
\]
Moreover,
\[
\sum_{t=d+1}^{T}\frac{1}{t}
\ge
\int_{d+1}^{T+1}\frac{dx}{x}
=
\log\frac{T+1}{d+1}.
\]
Therefore,
\[
\mathbb{E}[R_T]\ge d\log\frac{T+1}{d+1}.
\]
Since the constructed IO instance satisfies the above lower bounds for every
$J\in\mathcal J$, taking the infimum over $J\in\mathcal J$ and then the
supremum over $(\theta^\star,\mathbb{P}_{\mathcal S})$ gives the claimed
result.

\end{document}